\documentclass{article}

\usepackage{url}            
\usepackage{booktabs}       
\usepackage{amsfonts}       
\usepackage{nicefrac}       
\usepackage{microtype}      

\usepackage{bbm}
\usepackage{amssymb}
 \usepackage{amsthm}
\usepackage{a4wide}
\usepackage{mathtools}
\usepackage{natbib}
\usepackage[normalem]{ulem}
 \usepackage[]{algorithm2e}
\usepackage{enumitem}
\usepackage{bbm}
\usepackage{graphicx}

\usepackage[backgroundcolor=White,textwidth=0.8in]{todonotes}

 \newtheorem{definition}{Definition}
 \newtheorem{remark}{Remark}
 \newtheorem{proposition}{Proposition}
 \newtheorem{theorem}{Theorem}
 \newtheorem{lemma}{Lemma}
 \newtheorem{corollary}{Corollary}

\newcommand{\X}{\mathcal{X}}
\newcommand{\Y}{\mathcal{Y}}

\newcommand{\indi}{\mathbbm{1}}

\newcommand{\Gds}{G_{\Delta}^*}

\DeclareMathOperator*{\argmax}{arg\,max}

\newcommand{\lp}{\left(}
\newcommand{\rp}{\right)}
\newcommand{\lb}{\left [}
\newcommand{\rb}{\right]}

\newcommand{\gd}{\gamma_{\delta}}  
\newcommand{\mbb}[1]{\mathbb{#1}}

\newcommand{\xhi}{x_{h,i}}
\newcommand{\xhit}{x_{h_t,i_t}}
\newcommand{\ut}{u_t(x_{h,i})}
\newcommand{\lt}{l_t(x_{h,i})}
\newcommand{\Xtu}{\mathcal{X}_t^{(u)}}
\newcommand{\Xtc}{\mathcal{X}_t^{(c)}}
\newcommand{\Xtd}{\mathcal{X}_t^{(d)}}
\newcommand{\It}[1]{I_t^{(#1)}}
\newcommand{\Et}[1]{E^{(t)}_{(#1)}}
\newcommand{\xtj}[1]{x^{(t)}_{(#1)}}
\newcommand{\Xhi}{\X_{h,i}}
\newcommand{\hgone}{\hat{\gamma}_1^{(t)}}
\newcommand{\hgtwo}{\hat{\gamma}_2^{(t)}}

\newcommand{\gls}{G_\lambda^*} 

\newcommand{\pr}{\mathbb{P}}
\newcommand{\mc}[1]{\mathcal{#1}}
\newcommand{\indicator}{\mathbbm{1}}

\newcommand{\tbf}[1]{\textbf{#1}}

\newcommand{\rej}{\Delta}

\newcommand{\ubar}[1]{\text{\b{$#1$}}}
\newcommand{\heub}{\hat{\bar{\eta }}}
\newcommand{\helb}{\hat{\ubar{\eta}}}

\title{Active Learning for Binary Classification with Abstention}

    \author{Shubhanshu Shekhar \\ \texttt{shshekha@eng.ucsd.edu} \and 
        Mohammad Ghavamzadeh \\ \texttt{mohammad.ghavamzadeh@inria.fr} \and 
        Tara Javidi \\ \texttt{tjavidi@eng.ucsd.edu}}

        \date{}
\begin{document}

\maketitle

\begin{abstract}

We construct and analyze active learning algorithms for the problem of binary classification with abstention. We consider three abstention settings: \emph{fixed-cost} and two variants of \emph{bounded-rate} abstention, and for each of them propose an active learning algorithm. All the proposed algorithms can work in the most commonly used active learning models, i.e., \emph{membership-query}, \emph{pool-based}, and \emph{stream-based} sampling. 
We obtain upper-bounds on the excess risk of our algorithms in a general non-parametric framework, and establish their minimax near-optimality by deriving matching lower-bounds. Since our algorithms rely on the knowledge of some smoothness parameters of the regression function, we then describe a new strategy to adapt to these unknown parameters in a data-driven manner. 
Since the worst case computational complexity of our proposed algorithms increases exponentially with the dimension of the input space, we conclude the paper with a computationally efficient variant of our algorithm whose computational complexity has a polynomial dependence over a smaller but rich class of learning problems.
\end{abstract}


\section{Introduction}
\label{sec:introduction}

We consider the problem of binary classification in which the learner has  an additional provision of abstaining from declaring a label. This problem models several practical scenarios in which it is preferable to withhold a decision, perhaps at the cost of some additional experimentation, instead of making an incorrect decision and incurring much higher costs. A canonical application of this problem is in automated medical diagnostic systems~\citep{rubegni2002automated}, where classifiers which defer to a human expert on uncertain inputs are more desirable than classifiers that always make a decision. Other key applications include dialog systems and detecting harmful contents on the web.

Several existing works in the literature, such as \citet{castro2008minimax, dasgupta2006coarse}, have demonstrated the benefits of active learning (under certain conditions) in standard binary classification. However, in the case of classification with abstention, the design of active learning algorithms and their comparison with their passive counterparts have largely been unexplored. In this paper, we aim to fill this gap in the literature. 
More specifically, we design active learning algorithms for  classification with abstention in three different settings. \tbf{Setting~1} is the \emph{fixed-cost} setting, in which every usage of the abstain option results in a known cost $\lambda \in (0,1/2)$.  \tbf{Setting~2} is the \emph{bounded-rate} with ``known" input marginal ($P_X$) setting. This provides a smooth transition from Setting~1 to Setting~3, and allows us to demonstrate the key algorithmic changes in this transition. \tbf{ Setting~3} is the \emph{bounded-rate} with ``unknown" marginal ($P_X$) setting. Here, the algorithm has the option to request $m$ additional unlabelled samples, so long as $m$  grows only polynomially with  the label budget $n$. The fixed-cost setting is suitable for problems where a precise cost can be assigned to additional experimentation due to using the abstain option. In applications such as medical diagnostics, where the bottleneck is the processing speed of the human expert \citep{pietraszek2005optimizing}, the bounded-rate framework is more natural.

\tbf{Prior Work:} \cite{chow1957optimum} studied the problem of passive learning with abstention and derived the Bayes optimal classifier for both fixed-cost and bounded-rate settings (under certain continuity assumptions).~\citet{Chow1970:Optimum} further analyzed the trade-off between error rate and rejection rate. Recently, a collection of papers have revisited this problem in the fixed-cost setting.~\citet{herbei2006classification} obtained convergence rates for classifiers in a non-parametric framework similar to our paper.~\citet{bartlett2008classification} and~\citet{yuan2010classification} studied convex surrogate loss functions for this problem and obtained bounds on the excess risk of empirical risk minimization based classifiers.~\citet{wegkamp2007lasso} and~\citet{wegkamp2011support} studied an $\ell_1$-regularized version of this problem.~\citet{cortes2016learning} introduced a new framework which involved learning a pair of functions and proposed and analyzed convex surrogate loss functions. The problem of binary classification with a \emph{bounded-rate} of abstention has also been studied, albeit less extensively.~\citet{pietraszek2005optimizing} proposed a method to construct abstaining classifiers using ROC analysis.~\citet{denis2015consistency} re-derived the Bayes optimal classifier for the bounded rate setting under the same assumptions as~\citet{chow1957optimum}. They further proposed a general plug-in strategy for constructing abstaining classifiers in a semi-supervised setting, and obtained an upper bound on the excess risk. 

\tbf{Contributions:} For each of the three abstention setting mentioned earlier, we propose an algorithm that can work with three common active learning models~\citep[\S~2]{settles2009active}: {\em membership query}, {\em pool-based}, and {\em stream-based}. After describing the algorithms, we obtain upper-bounds on their excess risk in a general non-parametric framework with mild assumptions on the joint distribution of input features and labels (Section~\ref{sec:algorithms}). The obtained rates compare favorably with the existing results in the passive setting thus characterizing the gains associated with active learning (see Section~\ref{sec:discussion} for a discussion). Since our proposed algorithms require knowledge of certain smoothness parameters, in Section~\ref{sec:adaptivity}, we propose a new adaptive scheme that adjusts to the unknown smoothness terms in a data driven manner. In Section~\ref{sec:lower}, we derive lower-bounds on the excess risk for both fixed cost and bounded rate settings to establish the minimax near-optimality of our algorithms. Finally, we conclude in  Section~\ref{sec:practical} by describing a computationally feasible version of our algorithm for a restricted but rich class of problems. 








\section{Preliminaries}
\label{sec:preliminaries}
Let $\X$ denote the input space and $\Y = \{0,1\}$ denote the set of labels to be assigned to points in $\X$. We assume\footnote{This is to simplify the presentation; our work can be readily extended to any compact metric space $(\X, d)$ with finite metric dimension.} that $\X = [0,1]^D$ and $d$ is the Euclidean metric on $\X$, i.e.,~for all $x,x'\in\X$,~$d(x,x'):=\sqrt{\sum_{i=1}^D(x_i-x'_i)^2}$.
A binary classification problem is completely specified by $P_{XY}$, i.e.,~the joint distribution of the input-label random variables. Equivalently, it can  also be represented in terms of the marginal over the input space, $P_X$, and the regression function $\eta(x) \coloneqq P_{Y|X}\lp Y=1 \mid X=x\rp$. 
A (randomized) abstaining classifier is defined as  a mapping  $g:\X \mapsto \mathcal{P}\lp \Y_1\rp$, where $\Y_1 =\Y \cup \{\rej\}$, the symbol $\rej$ represents the option of the classifier to abstain from declaring a label, and $\mathcal{P}(\Y_1 )$ represents the set of probability distributions on $\Y_1$. Such a classifier $g$ comprises of three functions $g_i:\X \to [0,1]$, for $i \in \Y_1$, satisfying $\sum_{i \in \Y_1} g_i(x) = 1$, for each $x \in \X$. A classifier $g$ is called \emph{deterministic} if the functions $g_i$ take values in the set $\{0,1\}$. Every deterministic classifier $g$ partitions the input set $\X$ into three disjoint sets $(G_{0},G_1,G_{\rej})$. 

Two common abstention models considered in the literature are:
\begin{itemize}[leftmargin=*]
\item \tbf{Fixed Cost}, in which the abstain option can be employed with a fixed cost of $\lambda \in (0,1/2)$. In this setting, the classification risk is defined as $l_{\lambda}(g,x,y) \coloneqq \indi_{\{g(x)\neq \rej\}}\indi_{\{g(x) \neq y\}} + \lambda \indi_{\{g(x)=\rej\}}$, and the classification problem is stated as 
\begin{equation}
\label{fixed_cost_problem}
  \min_g  R_\lambda(g) = \mathbb{E}[l_{\lambda}(g,X,Y)]= P_{XY}\big( g(X)\neq Y\ , g(X) \neq \rej \big) + \lambda P_X\big(g(X) = \rej\big).
  \end{equation} 
The Bayes optimal classifier is defined as $g^*_\lambda(x) = 1$, $0$, or $\rej$, depending on whether $1-\eta(x)$, $\eta(x)$, or $\lambda$ is the smallest. 

\item \tbf{Bounded-Rate}, in which the classifier can abstain up to a fraction $\delta \in (0,1)$ of the input samples. In this setting, we define the misclassification risk of a classifier $g$ as 
    $R(g) \coloneqq P_{XY}\big( g(X)\neq Y\ , g(X) \neq \rej \big)$,  
and state the classification problem as 
\begin{equation}
\begin{aligned}
 \underset{g}{\text{$\min$}} \;\;\; R(g), \qquad\qquad \text{subject to} \quad P_X\big(g(X) = \rej\big) \leq \delta.
\end{aligned}
\label{bayes_optimal_delta}
\end{equation}
The Bayes optimal classifier for~\eqref{bayes_optimal_delta} is in general a randomized classifier. However, under some continuity assumptions on the joint distribution $P_{XY}$, it is again of a threshold type, $g^*_\delta(x) = 1$, $0$, or $\rej$, depending on whether $1-\eta(x)$, $\eta(x)$, or $\gd$ is minimum, where $\gd \coloneqq \sup\{\gamma\geq 0\ \mid\ P_X(|\eta(X)-1/2|\leq \gamma)\leq \delta \}$. 
\end{itemize}

The main difference between~\eqref{fixed_cost_problem} and~\eqref{bayes_optimal_delta} is that in the fixed cost setting, the threshold levels are known beforehand, while in the bounded rate of abstention setting, the mapping $\delta \mapsto \gd$ is not known, and in general is quite complex. In order to construct a classifier that satisfies the constraint in \eqref{bayes_optimal_delta}, we need some information about the marginal $P_X$. Accordingly, we consider two variants of the bounded-rate setting: (i) the marginal $P_X$ is completely known to the learner, and (ii) $P_X$ is not known, and the  learner can request a limited number (polynomial in query budget $n$) of unlabelled samples  to estimate the measure of any set of interest. 



\vspace{-1em}
\paragraph{Active learning models:}
For every abstention model mentioned above, we propose active learning algorithms that can work in three commonly used active learning settings~\citep[\S~2]{settles2009active}: (i) \emph{membership query synthesis}, (ii) \emph{pool-based}, and (iii) \emph{stream-based}. Membership query synthesis requires the strongest query model, in which the learner can request labels at any point of the input space. 
A slightly weaker version of this model is the pool-based setting, in which the learner is provided with a pool of unlabelled samples and must request labels of a subset of the pool. Finally, in the stream-based setting, the learner receives a stream of samples and must decide whether to request a label or discard the sample.

\subsection{Definitions and Assumptions}
\label{subsec:definitions}

To construct our classifier, we will require a hierarchical sequence of partitions of the input space, called the tree of partitions~\citep{bubeck2011x,munos2014bandits}. 

\begin{definition}
\label{definition:tree}
A sequence of subsets $\{ \X_h\}_{h \geq 0}$ of $\X$ are said to form a \emph{tree of partitions} of $\X$, if they satisfy the following properties: \textbf{(i)} $|\X_h| = 2^h$ and we denote the elements of $\X_h$ by $x_{h,i}$, for $1 \leq i \leq 2^h$, 
\textbf{(ii)} for every $\xhi \in \X_h$, we denote by $\X_{h,i}$, the \emph{cell} associated with $\xhi$, which is defined as $\X_{h,i} \coloneqq \{ x \in \X \mid \, d(x, \xhi) \leq d(x, x_{h,j}),\;\forall j \neq i\}$, where ties are broken in an arbitrary but deterministic manner, and \textbf{(iii)} there exist constants $0<v_2\leq 1 \leq v_1$ and $\rho \in (0,1)$ such that for all $h$ and $i$, we have $B(\xhi, v_2\rho^h) \subset \X_{h,i} \subset B(\xhi, v_1\rho^h)$, where $B(x,a) \coloneqq \{ x' \in \X \ \mid \ d(x,x') < a \}$ is the open ball in $\X$ centered at $x$ with radius $a$. 
\end{definition}

\begin{remark}
\label{param-select0}
For the metric space $(\X,d)$ considered in our paper, i.e.,~$\X=[0,1]^D$ and $d$ being the Euclidean metric, the cells $\Xhi$ are $D$ dimensional rectangles. Thus, a suitable choice of parameter values for our algorithms are $\rho=2^{-1/D}$, $v_1=2\sqrt{D}$, and $v_2=1/2$.
\end{remark}

Next, we define the dimensionality of the region of the input space at which the regression function $\eta(\cdot)$ is close to some threshold value $\lambda$. 

\begin{definition}
\label{def:dimension1}
For a function $\zeta:[0,\infty) \mapsto [0,\infty)$ and a threshold $\lambda \in (0,1/2)$, we define the near-$\lambda$ dimension associated with $(\X, d)$ and the regression function $\eta(\cdot)$ as 
\begin{equation}
    \label{eq:dimension1}
    D_{\lambda}\lp \zeta\rp \coloneqq \inf \big\{a\geq 0\ \mid\ \ \exists C>0: \  M\big(\X_\lambda\big(\zeta(r)\big), r\big) \leq Cr^{-a}, \; \forall r>0\big\}, 
\end{equation}
where $\X_{\lambda}\big(\zeta(r)\big) \coloneqq \big\{x \in \X \ \mid \ |\eta(x) - \lambda| \leq \zeta(r)\big\}$ and $M(S,r)$ is the $r$ packing number of $S\subseteq \X$.
\end{definition}

The above definition is motivated by similar definitions used in the bandit literature such as the \emph{near-optimality dimension} of~\citet{bubeck2011x} and the \emph{zooming dimension} of~\citet{kleinberg2013bandits}. For the case of $\X=[0,1]^D$ considered in this paper, the term $D_{\lambda}(\zeta)$ must be no greater than $D$, i.e.,~$D_{\lambda}(\zeta) \leq D$. This is because $\X_{\lambda}\big(\zeta(r)\big) \subset \X$, for all $r>0$, and there exists a constant $C_D<\infty$, such that $M(\X,r) \leq C_Dr^{-D}$, for all $r>0$. 

\begin{remark}
\label{remark:dimension}
We will use an instance of near-$\lambda$ dimension for stating our results defined as $\tilde{D} = \max_{j=1,2}\{ \tilde{D}_j\}$, where $\tilde{D}_j \coloneqq D_{1/2 + (-1)^j(1/2-\lambda)}\lp \zeta_1 \rp$ and $\zeta_1(r) = 12(L_1v_1/v_2)^\beta r^\beta$, for $r>0$. 
\end{remark}

\vspace{-1em}
\paragraph{Assumptions:} We now state the assumptions required for the analysis of our classifiers: 

\begin{description}
\vspace{-0.6em}
\item [(MA)\label{assump:margin}]  
The joint distribution $P_{XY}$ of the input-label pair satisfies 
the {\em margin assumption} with parameters $C_0 > 0$ and $\alpha_0 \geq 0$, for $\gamma$ in the set $\{1/2-\gd, 1/2 + \gd \}$, which means that for any $0<t\leq 1$, we have $P_X\lp |\eta(X) - \gamma| \leq t\rp \leq C_0t^{\alpha_0}$, for $\gamma \in \{1/2-\gd, 1/2+\gd\}$.     

\item [(H\"O)\label{assump:holder}]
The regression function $\eta$ is H\"older continuous with parameters $L>0$ and $0<\beta \leq 1$, i.e.,~for all $x_1, x_2 \in \lp \X, d\rp$, we have $|\eta(x_1) - \eta(x_2)| \leq L\d(x_1,x_2)^{\beta}$.

\item [(DE)\label{assump:detect}]
For the values of $\gamma$ in the same set as in (MA), we define the {\em detectability assumption} with parameters $C_1>0$ and $\alpha_1\geq \alpha_0$ as $P_X\lp |\eta(X) - \gamma| \leq t\rp \geq C_1t^{\alpha_1}$, for any $0<t\leq 1$.


\end{description}

The (MA) and (H\"O) assumptions are quite standard in the nonparametric learning literature~\citep{herbei2006classification, minsker2012plug}. The (DE) assumption, which is only required in the bounded-rate setting, has also been employed in several prior works such as~\citet{castro2008minimax, tong2013plug}.
A detailed discussion of these assumptions is presented in Appendix~\ref{appendix:assumptions}

\section{Active Learning Algorithms}
\label{sec:algorithms}

We consider three settings for the problem of binary classification with abstention in this paper. For each setting, we propose an active learning algorithm and prove an upper-bound on its excess risk. 

The algorithm for Setting~1 provides us with the general template which is also followed in the other two settings with some additional complexity. Because of this, we describe the specifics of the algorithm for Setting~1 in the main text, and relegate the details of the algorithmic as well as analytic modifications required for Settings~2 and~3 to the appendix. Throughout this paper, we will refer to the algorithm for Setting~$j$ as Algorithm~$j$, for $j=1$, $2$, and $3$.


\subsection{Setting~1: Abstention with the fixed cost $\lambda \in (0,1/2)$}
\label{subsec:setting1}

In this section, we first provide an outline of our active learning algorithm for this setting (Algorithm~1). We then describe the steps of this algorithm and present an upper-bound on the excess risk of the classifier constructed by the algorithm. We report the pseudo-code of the algorithm and the proofs in Appendices~\ref{appendix:setting1-algo} and~\ref{appendix:setting1-proofs}.

\vspace{-1em}
\paragraph{Outline of Algorithm~1.} 
At any time $t$, the algorithm maintains a set of active points $\X_t \subset \cup_{h \geq 0}\X_h$, such that the cells associated with the points in $\X_t$ partition the whole $\X$, i.e.,~$\cup_{\xhi \in \X_t}\X_{h,i} = \X$.
The set $\X_t$ is further divided into \emph{classified} active points, $\X_t^{(c)}$, \emph{unclassified} active points, $\Xtu$, and \emph{discarded} points, $\Xtd$. The classified points are those at which the value of $\eta(\cdot)$ has been estimated sufficiently well so that we do not need to evaluate them further. The unclassified points  require further evaluation and perhaps refinement before making a decision. The discarded points are those for which we do not have sufficiently many unlabelled samples in their cells (in the \emph{stream-based} and \emph{pool-based} settings). For every active point, the algorithm computes high probability upper and lower bounds on the maximum and minimum $\eta(\cdot)$ values in the cell associated with the point. The difference of these upper and lower bounds can be considered as a surrogate for the uncertainty in the $\eta(\cdot)$ value in a cell. In every round, the algorithm selects a candidate point from the unclassified set that has the largest value of this uncertainty. 
Having chosen the candidate point, the algorithm either refines the cell or asks for a label at that point. 

\vspace{-1em}
\paragraph{Steps of Algorithm~1.}
The algorithm proceeds in the following steps:
\vspace{-0.7em}
\begin{enumerate}[wide, labelwidth=!, labelindent=0pt]
    \item For $t=0$, initialize $\X_0 = \{x_{0,0}\}$, $\X_0^{(u)}=\X_0$, $\X_0^{(c)} = \emptyset$, $\X_0^{(d)}=\emptyset$, $u_0(x_{0,0}) = +\infty$, and $l_0(x_{0,0}) = -\infty$. 

    \item For $t\geq1$, for every $\xhi \in \X_t$, we calculate $\ut$ and $\lt$, which are an upper-bound on the maximum value and a lower-bound on the minimum value of the regression function $\eta(\cdot)$ in $\X_{h,i}$, respectively. We define $u_t(\xhi) = \min \{\bar{u}_t(\xhi), u_{t-1}(\xhi)\}$, where $\bar{u}_t(\xhi) = \hat{\eta}_t(\xhi) + e_t(n_{h,i}(t)) + V_h$. Here  $\hat{\eta}_t(\xhi)$ is the empirical estimate of $\eta$ in the cell $\X_{h,i}$, $n_{h,i}(t)$ is the number of times the cell $\X_{h,i}$ has been queried by the algorithm up to time $t$, $e_t(n_{h,i}(t))$ represents the confidence interval length at $\xhi$ (see Lemma~\ref{lemma:step_1_1} in Appendix~\ref{appendix:setting1-proofs}), and $V_h=L(v_1\rho^h)^\beta$ is an upper-bound on the maximum variation of the regression function $\eta(\cdot)$ in a cell at level $h$ of the tree of partitions. The term $l_t(\xhi)$ is defined in a similar manner using $\max$ instead of $\min$ and using $\bar{l}_t(\xhi) = \hat{\eta}_t(\xhi) - e_t(n_{h,i}(t)) - V_h$. We add all points $\xhi \in \X_t$ to the set $\Xtc$, if they satisfy any one of these three conditions, {\bf\em (a)} $u_t(\xhi) < \lambda $, {\bf\em (b)} $l_t(\xhi) > 1-\lambda$, or {\bf\em (c)} $\lambda < l_t(\xhi) < u_t(\xhi) < 1-\lambda$. 
    
    \item The set of unclassified active points, $\Xtu$, are those points in $\X_t$ for which $[\lt, \ut]\cap\{\lambda, 1-\lambda\}$ is nonempty. 
    
    \item We select a candidate point $\xhit$ from $\Xtu$ according to the rule $\xhit \in \argmax_{\xhi \in \Xtu} \ I^{(1)}_t(x_{h,i})$, 
    %
    %
    where we define the index $\It{1}(\xhi) \coloneqq \ut - \lt$. 
   
    \item Once a candidate point $\xhit$ is selected, we take one of the following two actions: 
    
    \begin{enumerate}
   \item \emph{Refine.} If the uncertainty in the regression function value at $\xhit$, denoted by $e_t(n_{h_t,i_t}(t))$, is smaller than the upper-bound on the function variation in the cell $\X_{h_t,i_t}$, denoted by $V_{h_t} = L(v_1\rho^{h_t})^\beta$, and if $h_t \leq h_{\max}$, then we perform the following operations:
   \begin{align*}
  &\X_t \leftarrow \big(\X_t \setminus \{x_{h_t,i_t}\}\big) \cup \{x_{h_t+1, 2i_t-1}, x_{h_t +1, 2i_t}\}, \qquad u_t(x_{h_t+1,2i_t-1}) = u_t(x_{h_t,i_t}), \\ 
  &l_t(x_{h_t+1,2i_t-1}) = l_t(x_{h_t,i_t}), \qquad u_t(x_{h_t+1,2i_t}) = u_t(x_{h_t,i_t}), \qquad l_t(x_{h_t+1,2i_t}) = l_t(x_{h_t,i_t}). 
   \end{align*}
   \item \emph{Request a Label.} Otherwise, for each active learning model, we proceed as follows:
   \begin{itemize}[leftmargin=*]
       \item In the \emph{membership query model}, we request for the label at any point in the cell $\X_{h_t,i_t}$ associated with $\xhit$. 
       \item In the \emph{pool-based model}, we request the label if there is an unlabelled sample remaining in the cell $\mc{X}_{h_t,i_t}$. Otherwise, we remove $\xhit$ from $\Xtu$, add it to $\Xtd$, and return to Step~2. 
       \item In the \emph{stream-based model}, we discard the samples until a point in the cell $\mc{X}_{h_t,i_t}$ arrives. If $N_n = 2n^2\log(n)$ samples have been discarded, we remove $\xhit$ from $\Xtu$, add it to $\Xtd$, and return to Step~2 without requesting a label.
   \end{itemize}
    \end{enumerate}
    
  \item Let $t_n$ denote the time at which the $n$'th query is made and the algorithm halts. Then, we define the final estimate of the regression function as $\hat{\eta}(x) = \hat{\eta}_{t_n}\big(\pi_{t_n}(x)\big)$, where 
  \begin{equation}
  \label{eq:temp000}
  \pi_{t_n}(x) \coloneqq \big\{ x_{h,i} \in \X_{t_n}\ \mid \ d(x,x_{h,i}) \leq d(x,x_{h',i'}), \; \forall x_{h',i'} \in \X_{t_n}\big\},
  \end{equation}
  and define the discarded region of the input space as $\tilde{\X}_n \coloneqq  \cup_{\xhi \in \X_{t_n}^{(d)}} \X_{h,i}$. 
  \item Finally, the classifier returned by the algorithm is defined as 
  \begin{equation}
  \label{eq:classifier-algo1}
      \hat{g}(x) = 
    \begin{cases}
    1 & \text{if } \; u_{t_n}\big( \pi_{t_n}(x)\big) >1-\lambda \; \text{ or } \; x \in  \tilde{\X}_n, \\
    0 & \text{if } \; l_{t_n}\big(\pi_{t_n}(x)\big) < \lambda \; \text{ and } \; x \not \in \tilde{\X}_n, \\
    \Delta & \text{otherwise}. 
    \end{cases}
  \end{equation}
  Note that the classifier~\eqref{eq:classifier-algo1} arbitrarily assigns label $1$ to the points in the discarded region $\tilde{\X}_n$. 
\end{enumerate}

\begin{remark}
\label{smoothness-param}
Algorithm~1 (and as we will see later Algorithms~2 and~3) assumes the knowledge of parameters $v_1$, $\rho$, $L$, and $\beta$.
As described in Remark~\ref{param-select0}, it is straightforward to select the parameters $v_1$ and $\rho$, but the smoothness parameters $L$ and $\beta$ are often not   known to the algorithm.
We address this in Section~\ref{sec:adaptivity} by designing an algorithm that adapts to the smoothness parameters. 
\end{remark}


In the membership query model, the discarded set remains empty since the learner can always obtain a labelled sample from any cell. We begin with a result that shows that even in the other two models, the probability mass of the discarded region is small under some mild assumptions. 

\begin{lemma}
\label{lemma:discarded_region}
Assume that in the pool-based model, the pool size $M_n$ is greater than $\max\{ 2n^3, 16n^2 \log(n)\}$ and in the stream-based model, the term $N_n$ is set to $2n^2\log(n)$. Then, we have $\pr\big(P_X(\tilde{\X}_n) > 1/n\big) \leq 1/n$. 
\end{lemma}

This lemma (proved in Appendix~\ref{appendix:discarded-region-lemma}) implies that in the pool-based and stream-based models, with high probability, the misclassification risk of $\hat{g}$ can be upper-bounded by $1/n + P_{XY}\big(\hat{g}(X) \neq Y,\; \hat{g}(X) \neq \Delta,\; X \not \in \tilde{\X}_n\big)$. Lemma~\ref{lemma:discarded_region} is quite important because it implies that under some mild conditions, the analysis of the pool-based and stream-based models reduces to the analysis of the membership query model with an additional cost that can be upper bounded by $1/n$.

We now prove an upper-bound on the excess risk of the classifier (see Appendix~\ref{appendix:setting1-proofs} for the proof). 

\begin{theorem}
\label{theorem:setting1}
Suppose that the assumptions~(MA) and~(H\"O) hold, and  
let $\tilde{D}$ be the dimension term defined in Remark~\ref{remark:dimension}. 
Then, for large enough $n$, with probability at least $1-2/n$, for the classifier $\hat{g}$ defined by~\eqref{eq:classifier-algo1} and for any $a>\tilde{D}$, we have
%
%
\begin{equation}
\label{eq:theorem1_eq2}
R_\lambda(\hat{g}) - R_\lambda(g^*_\lambda ) = \tilde{\mathcal{O}}\big( n^{-\beta(\alpha_0+1)/(2\beta + a)} \big),
\end{equation}
where the hidden constant depends on the parameters $L$, $\beta$, $v_1$, $v_2$, $\rho$, $C_0$, and $a$.     
%
%
\end{theorem}
%

The above result improves upon the convergence rate of the plug-in scheme of \cite{herbei2006classification} in the passive setting mirroring the benefits of active learning in the standard binary classification problems. See Section~\ref{sec:discussion} and Appendix~\ref{appendix:discussion} for further discussion. 

\subsection{Setting~2: Bounded-rate setting with known $P_X$}
\label{subsec:setting2}

This setting provides an intermediate step between the fixed-cost and bounded-rate settings. The key difference between the algorithms for this and the \emph{fixed-cost} setting lies in the rule used for updating the set of unclassified points. Since in this case the threshold is not known, we need to use the current estimate of the regression function to obtain upper and lower bounds on the true threshold, and then use these bounds to decide which parts of the input space have to be further explored. We report the details of the algorithm in Appendix~\ref{appendix:setting2-details}, its pseudo-code in Appendix~\ref{appendix:setting2-algo}, and the statement and proof of its excess risk bound (Theorem~\ref{theorem:setting2}) in Appendix~\ref{appendix:setting2-proofs}.


\subsection{Setting~3: Bounded-rate setting with unlabelled samples}
\label{subsec:setting3}

Finally, we consider the general bounded-rate abstention model in the semi-supervised setting. In this case, the algorithm should request for unlabelled samples and use them to both construct the estimates of the appropriate threshold values and obtain better empirical estimates of the $P_X$ measure of a set. Unlike Algorithm~2, in Algorithm~3 we have to construct estimates of the threshold using empirical measure $\hat{P}_X$, and furthermore, based on the error in estimate of $\eta(\cdot)$, we also need a strategy of updating $\hat{P}_X$ by requesting more unlabelled samples. We report the details of Algorithm~3 in Appendix~\ref{appendix:setting3-details}, its pseudo-code in Appendix~\ref{appendix:setting3-algo}, and the statement and proof of its excess risk bound (Theorem~\ref{theorem:setting3}) in Appendix~\ref{appendix:setting3-proofs}. We note that the excess risk bound for Algorithm~3 is minimax (near)-optimal under the same assumptions as in Algorithms~1 and~2. However, in order to exploit easier problem instances in which $\tilde{D}$ is much smaller than $D$, we require an additional (DE) assumption (see Section~\ref{sec:discussion} for detailed discussion). 


\section{Adaptivity to Smoothness Parameters}
\label{sec:adaptivity}
 All the active learning algorithms discussed in Section~\ref{sec:algorithms} assume the knowledge of the H\"older smoothness parameters $L$ and $\beta$. We now present a simple strategy to achieve adaptivity to these parameters. To simplify the presentation, we only consider the problem in the fixed-cost setting with membership query model. Extension to the other settings and models could be done in the same manner. The parameters $(L, \beta)$ are required  by Algorithm~1 at two junctures: \tbf{1)} to define the index $I_t^{(1)}$ for selecting a candidate point, and \tbf{2)} to decide when to refine a cell. 
In our proposed adaptive scheme, we address these issues as follows:
\begin{itemize}[leftmargin=*]
\item Instead of selecting one candidate point in each step, we select one point from each level $h$ from the current set of active points. This is similar to the approach used in the SOO algorithm~\citep{munos2011optimistic} for global optimization. Since the maximum depth of the tree $h_{\max}$ is $\mc{O}\lp \log n \rp$, this modification only results in  an additional $\text{poly}\log n$ factor in the excess risk. 

\item 
To decide when to refine, we need to estimate the variation of $\eta(\cdot)$ in a cell from samples. We make an additional assumption, (QU), that the pair $\lp (\X_h)_{h \geq 0}, \eta \rp$  has \emph{quality} $q>0$ (see Appendix~\ref{appendix:adaptivity} for the definition). This assumption has been used in prior works on adaptive global optimization~\citep{slivkins2011multi, bull2015adaptive}. We then proceed by proposing a local variant of Lepski's technique~\citep{lepski1997optimal} to construct the required estimate of the variation of $\eta(\cdot)$, combined with an appropriate stopping rule. 
\end{itemize}

With these two modifications and the additional \emph{quality} assumption (QU), we can achieve the rate $\tilde{\mc{O}}\lp n^{-\beta(1+ \alpha_0)/(2\beta + a)} \rp$, with $a> \tilde{D}$, thus, matching the performance of Algorithm~1. The details of the adaptive scheme and the proof of convergence rate are provided in Appendix~\ref{appendix:adaptivity}. 

\begin{remark}
\label{remark:adaptivity}
We note that there are other adaptive schemes for active learning, such as ~\citet{minsker2012plug, locatelli2017adaptivity}, that can also be applied to the problem studied in this paper. Our proposed adaptive scheme provides an alternative to these existing methods. Furthermore, our scheme can also be applied to classification problems with \emph{implicit} similarity information, similar to~\citet{slivkins2011multi}, as well as to problems with spatially inhomogeneous regression functions. 
\end{remark}

\vspace{-1em}
\section{Lower Bounds}
\label{sec:lower}
We now derive minimax lower-bounds on the expected excess risk in the fixed-cost setting and for the membership query model. Since this is the strongest active learning query model, the obtained lower-bounds are also true for the other two models. The proof follows the general outline for obtaining lower bounds  described in existing works, such as~\citet{audibert2007fast, minsker2012plug}, reducing the estimation problem to that of an appropriate multiple hypothesis testing problem, and applying Theorem~2.5 of~\citet{tsybakov2008introduction}. 
The novel elements of our proof are the construction of an appropriate class of regression functions (see Appendix~\ref{appendix:lower_bound}) and the comparison inequality presented in Lemma~\ref{lemma:lower_bound1}. 

We begin by presenting a lemma that provides a lower-bound on the excess risk of an abstaining classifier in terms of the probability of the mismatch between the abstaining regions of the given classifier and the Bayes optimal classifier. The proof of Lemma~\ref{lemma:lower_bound1} is given in Appendix~\ref{appendix:lower_bound}.

\begin{lemma}
\label{lemma:lower_bound1}
In the fixed-cost abstention setting with cost of abstention equal to $\lambda<1/2$, let $g$ represent any abstaining classifier and $g^*_\lambda$ represent the Bayes optimal one. Then, we have 
\begin{equation}
\label{eq:comparison_lemma}
 R_\lambda\lp g \rp - R_\lambda \lp g_\lambda^*\rp \geq c P_X\bigg( (G^*_\lambda \setminus
G_\lambda) \cup (G_\lambda \setminus G^*_\lambda) \bigg)^{\frac{1+\alpha_0}{\alpha_0}}, 
\end{equation}
where 
$c>0$ is a constant, and $\alpha_0$ is the parameter used in the assumptions of Section~\ref{subsec:definitions}.
\end{lemma}

Lemma~\ref{lemma:lower_bound1} aids our lower-bound proof in several ways: {\bf 1)} the RHS of~\eqref{eq:comparison_lemma} motivates our construction of \emph{hard} problem instances, in which it is difficult to distinguish between the `abstain' and `not-abstain' options, {\bf 2)} the RHS of~\eqref{eq:comparison_lemma} also suggests a natural definition of pseudo-metric (see Theorem~\ref{theorem:tsybakov1} in Appendix~\ref{appendix:lower_bound2}), and {\bf 3)} it allows us to convert the lower-bound on the hypothesis testing problem to that on the excess risk. We now state the main result of this section (see Appendix~\ref{appendix:lower_bound} for the proof). 

\begin{theorem}
\label{theorem:lower_bound2}
Let $\mc{A}$ be any active learning algorithm and $\hat{g}_n$ be the abstaining classifier learned by $\mc{A}$ with $n$ label queries in the fixed-cost abstention setting, with cost $\lambda<1/2$. Let $\mc{P}\lp L, \beta, \rho_0\rp$ represent the class of joint distributions $P_{XY}$ satisfying the margin assumption~(MA) with exponent $\alpha_0>0$, whose regression function is $(L, \beta)$ H\"older continuous with $L\geq 3$ and $0<\beta\leq 1$. Then, we have 
\begin{align*}
\inf_{\mc{A}} \sup_{P_{XY} \in \mc{P}(L, \beta, \alpha_0)} \bigg( \mbb{E}\lb R_\lambda\lp \hat{g}_n \rp - R_\lambda \lp 
g_\lambda^* \rp \rb  \bigg) & \geq C n^{-\beta(1+\alpha_0)/(2\beta + D)}. 
\end{align*}
\end{theorem}


Finally, by exploiting the relation between the Bayes optimal classifier in the fixed-cost and bounded-rate of abstention settings, we can obtain the following lower-bound on the expected excess risk in the bounded-rate of abstention setting.

\begin{corollary}
\label{corollary:lower_bound3}
For the bounded-rate of abstention setting, we have the following lower-bound:
\begin{align*}
    \inf_{\mc{A}} \sup_{P_{XY} \in \mc{P}\lp L, \beta, \alpha_0 \rp }  \lp \mbb{E} \lb R (\hat{g}_n) - R(g^*_{\delta}) \rb \rp \geq C n^{-\beta( 1+\alpha_0)/(2\beta + D)}. 
\end{align*}
\end{corollary}

The proof of this statement is given in Appendix~\ref{appendix:lower_bound}. 

\section{Computationally Feasible Algorithms}
\label{sec:practical}
The lower bound obtained in the previous section implies that in the worst case, to ensure an excess risk smaller than $\epsilon>0$, any algorithm will require $\Omega\lp (1/\epsilon)^{\frac{2\beta + D}{\beta(1+\rho_0)}}\rp$ label requests (in both the fixed-cost and bounded-rate settings). This means that the worst case computational complexity of any algorithm will have an exponential dependence of the dimension.  The above discussion suggests that to obtain computationally tractable algorithms, we need to restrict the hypothesis class. We consider the class of learning problems where the regression function is a generalized linear map given by $\eta(x) = \psi \lp \langle x, w^* \rangle \rp + 1/2 $ where $\psi:\mbb{R} \mapsto [-1/2,1/2]$ is a monotonic invertible $(L,\beta)$~H\"older continuous function. This class of problems (henceforth denoted by $\mc{P}_1(L,\beta,\rho_0)$), though much smaller than $\mc{P}(L, \beta, \rho_0)$ considered in previous sections, contains standard problem instances such as linear classifiers and logistic regression. Furthermore, by using appropriate feature maps, the class $\mc{P}_1\lp L, \beta, \rho_0\rp$ can model very complex decision boundaries. 

Due to the special structure of the regression function, the learning problem (for Setting~1) then reduces to estimating the optimal hyperplane $w^*$, and the value $\psi^{-1}(\lambda)$. Here we can employ the \emph{dimension coupling} technique of \citet{chen2017near}, which implies that  the $D$ dimensional problem can be reduced to $D-1$ two dimensional problems. Furthermore, as we show in Proposition~\ref{prop:computationally_feasible} (stated and proved in Appendix~\ref{appendix:discussion_feasible}), for an $\epsilon>0$ a modified version of Algorithm~1 can estimate the term $w^*$ for continuously differentiable $\psi$  with accuracy $\epsilon$ for a number of labelled samples which has a polynomial dependence of the dimension $D$.

\section{Discussion}
\label{sec:discussion}
\paragraph{Improved Convergence Rates (active over passive learning).} 
The convergence rates on the excess risk obtained by our active learning algorithms improve upon those in the literature obtained in the passive case. More specifically, the excess risk in the passive case for the fixed-cost~\citep{herbei2006classification} and bounded-rate~\citep{denis2015consistency} settings is $\mc{O} \lp n^{-\beta(1+\alpha_0)/(D + 2\beta + \alpha_0 \beta)}\rp $ (using the estimators of~\citealt{audibert2007fast}). In contrast, all our algorithms achieve an excess risk of $\mc{O}\lp n^{-\beta(1+\alpha_0)/(a + 2\beta)}\rp$, for $a>\tilde{D}$. Thus, even for the worst case of $\tilde{D}=D$, our algorithms achieve faster convergence in both abstention settings. 
Moreover, under the additional assumption that $P_X$ admits a density $p_X$ w.r.t.~the Lebesgue measure, such that $p_X \geq c_0 >0$, for all $x \in \X$, the convergence rates in the passive case for both abstention settings improve by getting rid of the $\beta \alpha_0$ term in the exponent. The performance of our algorithms also improves further with this additional assumption, and we can show that $\tilde{D} \leq \max\{0, D- \beta \alpha_0\}$ (see Appendix~\ref{appendix:discussion_improved_rates} for details).


\paragraph{Necessity of the Detectability (DE) Assumption.} 
In Setting~3, the size of the \emph{unclassified region}, $\cup_{x_{h,i} \in \Xtu}\X_{h,i}$, depends on two terms: {\bf 1)} the error in the estimate of the regression function $\eta(\cdot)$, and {\bf 2)} the error due to using the empirical measure $\hat{P}_X$. The (DE) assumption ensures that for sufficiently accurate empirical estimates of the marginal $P_X$, we can control the \emph{size} of the unclassified region in terms of the errors in the estimate of the regression function (similar to Settings~1 and~2). A situation, where without (DE), Algorithm~3 has to explore a much larger region of the input space than Algorithm~2 (in Setting~2) is given in Appendix~\ref{appendix:discussion_DE}. Since there exist problem instances for which $\tilde{D}=D$, we note that (DE) is not needed to match the worst-case performance of Algorithm~2. However, it is required in order to exploit the \emph{easy} problem instances with low values of $\tilde{D}$.

\vspace{-1em}
\section{Conclusions and Future Work}
\label{sec:conclucions}
In this paper, we proposed and analyzed active learning algorithms for three settings of the problem of binary classification with abstention.
The first setting considers the problem of classification with fixed cost of abstention, while the other settings consider two variants of classification with bounded abstention rate.
We obtained upper bounds on the excess risk of all the algorithms and  demonstrated their minimax (near)-optimality by deriving lower bounds. 
As all our algorithms relied on the knowledge of smoothness parameters, we then proposed a general strategy to adapt to these parameters in a data driven way. A novel aspect of our adaptive strategy is that it can also work for more general learning problems with implicit distance measure on the input space. Finally, we also presented a computationally efficient version of our algorithms for a small but rich class of problems. 

In Section~\ref{sec:practical}, we discussed an efficient version of our algorithms in the \emph{realizable} case when the Bayes optimal classifier is a halfspace. An important topic of ongoing research is to extend ideas presented in this paper to the \emph{agnostic} case, and design general computationally feasible active learning strategies for learning classifiers with abstention. 



\newpage
\bibliographystyle{apalike}
\bibliography{ref}


\newpage
\appendix



\newpage
\section{Details from Section~\ref{sec:introduction} and Section~\ref{sec:preliminaries}}
\label{appendix:details}

\subsection{Discussion on Assumptions}
\label{appendix:assumptions}

The \emph{margin assumption}~(MA) controls the amount of $P_X$ measure assigned to the regions of the input space with $\eta(\cdot)$ values in the vicinity of the threshold values.The assumption~(MA), which is a modification of the Tsybakov's margin condition for binary classification~\citep[Definition~7]{bousquet2003introduction}, has be employed in several existing works in classification with abstention literature such as \citep{herbei2006classification, bartlett2008classification, yuan2010classification}. 

The \emph{H\"older} continuity assumption ensures that points which are close to each other have similar distribution on the label set. For simplicity, we restrict our attention to the case of $\beta \leq 1$ so that it suffices to consider piecewise constant estimators. 
For H\"older functions with $\beta>1$, our algorithms can be suitably modified by replacing the piece-wise constant estimators with local polynomial estimators \citep[\S~1.6]{tsybakov2008introduction}. 

The \emph{detectability assumption}~(DE) is a converse of the (MA) assumption.
It provides a lower bound on the amount of $P_X$ measure in the regions of $\X$ with $\eta(\cdot)$ values close to the thresholds. 
We note that our proposed algorithms acheive the minimax optimal rates without this assumption. However, this assumption is required by our algorithm in the most general problem setting (Theorem~\ref{theorem:setting3}) for exploiting \emph{easier} problem instances.
Assumptions similar to (DE) have been used in various prior works in the nonparametric learning and estimation literature \citep{castro2008minimax, tong2013plug, rigollet2011neyman, cavalier1997nonparametric, tsybakov1997nonparametric}.
We discuss the necessity of this assumption in Section~\ref{sec:discussion} and in Appendix~\ref{appendix:discussion}.

\newpage
\section{Pseudo-code and Proofs of the Algorithm from Section~\ref{subsec:setting1}}
\subsection{Pseudo-code of Algorithm~1}
\label{appendix:setting1-algo}
In this section, we report the pseudo-code of Algorithm~1 that was outlined and described in Section~\ref{subsec:setting1}.
This is our active learning algorithm for the fixed-cost setting, with cost of abstention equal to $\lambda \in (0,1/2)$. As mentioned earlier our proposed algorithm can work in the three commonly used active learning frameworks, namely, \emph{membership query} model, \emph{pool-based} and \emph{stream-based} models. The only difference is the way the algorithm interacts with the labelling oracle, and this is captured by the REQUEST\_LABEL subroutine given in Appendix~\ref{appendix:request_label}.

\begin{algorithm}
\SetAlgoLined
\SetKwInput{Input}{Input}
\SetKwInput{Output}{Output}
\SetKw{Init}{Initialize}

\caption{Active learning algorithm for the fixed cost of abstention setting.}
\Input{$n$,$\lambda$, $L$, $\beta$, $v_1$,  $\rho$} 
 
\Init{$t=1$, $n_e=0$, $\X_t = \{x_{0,1}\}$, $\Xtu = \X_t$, $\Xtc= \emptyset$}
 
\If{$n_e = 0$}
{
$u_t(x_{0,1}) = + \infty$ \\
$l_t(x_{0,1}) = - \infty$ 
}
 \vspace{1em}
 \tcc{Remove the already classified points from the active set $\X_t$}
 \While{$n_e \leq n$}{

 \For{$x_{h,i}\in \Xtu$} 
 {
 $u_t(\xhi) \leftarrow \min \big\{ \bar{u}_t(\xhi), u_{t-1}(\xhi) \big\}$ \\
 $l_t(\xhi) \leftarrow \max \big\{ \bar{l}_t(\xhi), l_{t-1}(\xhi) \big\}$ \\
 
 \If{$[l_t(\xhi), u_t(\xhi)] \cap \{1/2-\gd, 1/2+\gd\}=\emptyset$}{
 $\Xtc \leftarrow \Xtc \cup\{\xhi\}$
 }

 }
 
 \vspace{1em}
 \tcc{Choose a candidate point with most uncertainty} 
 
 $x_{h_t,i_t} \in \argmax_{x_{h,i}\in \Xtu} \It{1}(\xhi) = u_t(\xhi) - l_t(\xhi)$\;

 \vspace{1em}
 
 \tcc{ Refine or Label}
 
 \eIf{$\;e_t\big(n_{h,i}(t)\big)< L(v_1\rho^{h_t})^{\beta}$}{
$\Xtu \leftarrow \Xtu \setminus \{\xhit\}\cup \{ x_{h_t+1, 2i_t-1}, x_{h_t+1, 2i_t}\}$ \\
$u_t(x_{h_t+1,2i_t-1})\leftarrow u_t(\xhit);\qquad$  $l_t(x_{h_t+1,2i_t-1})\leftarrow l_t(\xhit)$ \\
$u_t(x_{h_t+1,2i_t})\leftarrow u_t(\xhit);\qquad\;\;\;\;$ $l_t(x_{h_t+1,2i_t})\leftarrow l_t(\xhit)$ 
 }
 {
call REQUEST\_LABEL\\
 }
$t \leftarrow t+1$
}
 \Output{$\hat{g}$ defined by Eq.~\ref{eq:classifier-algo1}\\}
 \label{alg:basic_alg1}
 \end{algorithm}

\begin{algorithm}
\SetAlgoLined
\SetKwInput{Input}{Input}
\SetKwInput{Output}{Output}
\SetKw{Init}{Initialize}

\subsubsection{REQUEST\_LABEL Subroutine}
\label{appendix:request_label}

In the membership query mode, the algorithm can request label from some point in the cell corresponding to the point $\xhit$. In the pool based setting, the algorithm checks whether the currently unlabelled pool, denoted by $Z_t$ (i.e., the initial pool of samples with the points labelled by the algorithm before time $t$ removed), contains an element lying in the cell $\X_{h_t,i_t}$ or not. If there exists a point in $\X_{h_t,i_t}\cap Z_t$, then the algorithm requests a label at that point. Otherwise the cell $\X_{h_t,i_t}$ is discarded. Finally, in the stream based setting, the algorithm keeps rejecting points in the stream until a sample in $\X_{h_t,i_t}$ is observed, or if $N_n$ consecutive samples have passed. If a point lands in $\X_{h_t,i_t}$ then the algorithm requests its label, and if $N_n$ samples have been rejected, the algorithm discards the cell $\X_{h_t,i_t}$. 
\bigskip

\Input{Mode, $\xhit$} 
\vspace{0.4em}
Flag $\leftarrow$ False\;

  \uIf{Mode==`Membership'}{
 $y_t \sim \text{Bernoulli}(\eta(\xhit))$\;
 Increment $\leftarrow$ \texttt{True} \;
  }
  \vspace{1em}
  \uElseIf{Mode==`Pool'}{
 \tcc{Check if there is an unlabelled sample in the cell $\X_{h_t,i_t}$} 
    \eIf{$Z_t \cap \X_{h_t,i_t} \neq \emptyset$}{
    choose $\tilde{x}_{h_t,i_t} \in Z_t \cap \X_{h_t,i_t}$ arbitrarily \;  
    $y_t \sim \text{Bernoulli}\lp \eta\lp \tilde{x}_{h_t,i_t}\rp \rp$ \;
    $Z_t \leftarrow Z_t\setminus \{ \tilde{x}_{h_t,i_t}\}$\;
    Increment $\leftarrow$ True\;
    }
    { \tcc{ Otherwise discard the cell $\X_{h_t,i_t}$}
$    \Xtd \leftarrow \Xtd \cup\{\xhit\}$ \;
$\Xtu \leftarrow \Xtu \setminus \{\xhit\}$\;
    }

 } 
  
  \Else{
 counter $\leftarrow$ $1$ , discard $\leftarrow $ True,  Flag $\leftarrow$ True \;
  \While{$\big($counter $\leq N_n\big)$ AND Flag}{
  Observe next element of the stream $x \sim P_X$ \;
  \If{$x \in \X_{h_t,i_t}$}{
  $y_t \sim \text{Bernoulli}(\eta(x))$ \;
  discard $\leftarrow$ False, Increment $\leftarrow$ True \;
  Break
  }
 counter $\leftarrow$ counter $+1$\;
  }
  \If{discard}{
  $\Xtd \leftarrow \Xtd \cup \{\xhit \}$ \;
  $\Xtu \leftarrow \Xtu \setminus \{\xhit\}$\;
  } 
  
  \If{Increment}{
  \tcc{Increment the label request counter}
  $n_e \leftarrow n_e + 1$ \;
  }

  }
\vspace{0.5em}

\TitleOfAlgo{REQUEST\_LABEL}
\end{algorithm} 

\newpage
\subsection{Proof of Lemma~\ref{lemma:discarded_region}}
\label{appendix:discarded-region-lemma}

We begin with the proof of Lemma~\ref{lemma:discarded_region} which shows that with probability at least $1-1/n$, the $P_X$ measure of the (random) set $\tilde{\X}_n$ is no larger than $1/n$.

Suppose the discarded region $\tilde{\X}_n \coloneqq \cup_{\xhi \in \X_{t_n}^{(d)}}\X_{h,i}$ consists of $T$ components, i.e., $|\X_{t_n}^{(d)}| = T$. 
Since the algorithm only refines cells up to the depth $h_{\max} = \log(n)$, and the total number of cells in $\X_{h_{\max}}$ is $2^{h_{\max}} \leq e^{h_{\max}} = n$, we can trivially upper bound the number of discarded cells/points with $n$, i.e., $T \leq n$. 

\paragraph{Stream-based setting.}
 In this case a cell $\X_{h,i}$ is discarded, if after $N_n$ consecutive draws from $P_X$, none of the samples fall in $\X_{h,i}$.  We proceed as follows:
\begin{align*}
    \pr \lp P_X\lp \tilde{\X}_n \rp > 1/n \rp & = \pr \lp \sum_{\xhi \in \X_{t_n}^{(d)}} P_X\lp \X_{h,i}\rp >1/n\rp 
     \stackrel{(a)}{\leq} \pr \lp \exists \xhi \in \X_{t_n}^{(d)}\;:\; P_X\lp \X_{h,i}\rp > 1/(nT) \rp \\
     & \stackrel{(b)}{\leq} \sum_{\xhi \in \X_{t_n}^{(d)}} \pr \lp P_X\lp \X_{h,i}\rp > 1/(nT)\;; \; \xhi \in \X_{t_n}^{(d)} \rp \stackrel{(c)}{\leq} T\lp 1 - \frac{1}{nT}\rp^{N_n} \\
     & \stackrel{(d)}{\leq} n \lp 1 - \frac{1}{n^2}\rp^{N_n} \leq \exp\lp -\frac{N_n}{n^2} + \log\lp n\rp \rp \stackrel{(e)}{=} \frac{1}{n}.  
\end{align*}
In the above display,\\
\textbf{(a)} follows from the pigeonhole principle,  \\
\textbf{(b)} follows from an application of union bound, \\
\textbf{(c)} follows from the rule used for discarding cells in the stream-based setting, \\
\textbf{(d)} follows from the fact that $T \leq n$, and \\
\textbf{(e)} follows from the choice of $N_n = 2n^2\log(n)$.

\paragraph{Pool-based setting.}
Let $\mc{Z} = \{X_1, X_2, \ldots, X_{M_n}\}$ denote the pool of unlabelled samples available to the learner, and for any $\mc{X}_{h,i}$ we introduce the notation $M_{h,i} \coloneqq |\mc{Z} \cap \X_{h,i}|$ to represent the number of samples lying in the cell $\X_{h,i}$.  
Recall that a cell $\X_{h,i}$ is discarded if the number of unique unlabelled samples in the cell is smaller than the number of label requests in the cell, which can be trivially upper bounded by $n$, the total budget. 
Thus, introducing the terms $\mc{C}_1 \coloneqq \{\xhi \; \mid \; M_{h,i} < n \}$ and $\mc{C}_2 \coloneqq \{ \xhi \in \mc{C}_1 \; \mid \; P_X\lp \X_{h,i} \rp \geq 1/(n^2)\}$, we get the following (for any realization of $\mc{Z}$):
\begin{align*}
    P_X\lp \tilde{\X}_n \rp & \leq P_X\lp \bigcup_{\xhi \in \mc{C}_1 } \X_{h,i} \rp \leq n\lp \frac{1}{n^2}\rp + P_X\lp \bigcup_{\xhi \in \mc{C}_2} \X_{h,i} \rp, 
\end{align*}
where in first term after the second inequality above, we use the fact that the total number of cells discarded up to the depth of $\log(n)$ cannot be larger than $n$. 

Now, we claim that to complete the proof, it suffices to show that for any $\X_{h,i}$ such that $P_X\lp \X_{h,i} > 1/n^2\rp$, we have $\pr \lp M_{h,i}<n \rp \leq 1/n^2$.
This is because $\mc{C}_2 \subset \{ \xhi \; \mid \; P_X\lp \X_{h,i}\rp  \geq 1/n^2 \}$, and $|\mc{C}_2|\leq n$, and combined with the previous statement it implies that $\mc{C}_2$ is an empty set with proabability at least $1-1/n$. 

Consider any cell $\X_{h,i}$ such that $P_X(\X_{h,i}) = p \geq 1/n^2$. For points $X_j$ in $\mc{Z}$ define the $\text{Bernoulli}(p)$ random variable $U_j = \indi_{\{X_j \in \X_{h,i}\}}$. 
Suppose $M_n = \max\left \{ 2n^3, 16 n^2 \log(n) \right \}$. Then we have the following:
\begin{align*}
    \pr \lp M_{h,i}< n \rp & = \pr \lp \sum_{j=1}^{M_n}U_j < n\rp \stackrel{(a)}{\leq} \pr \lp \sum_{j=1}^{M_n} U_j < \frac{1}{2n^2} \rp \\
    &\stackrel{(b)}{\leq} \pr \lp \sum_{j=1}^{M_n}U_j \leq \lp 1 - 1/2\rp p \rp \stackrel{(c)}{\leq} \exp\lp -M_n p/8 \rp 
    \stackrel{(d)}{\leq} \frac{1}{n^2}. 
\end{align*}
In the above display:\\
\textbf{(a)} follows from the fact that $M_n \geq 2n^3$, \\
\textbf{(b)} follows from the fact that $p>1/n^2$, \\
\textbf{(c)} follows from the application of Chernoff inequality for the lower tail of Binomial, \\
\textbf{(d)} follows from the fact that $M_n \geq 16n^2 \log(n)$ and $p\geq 1/n^2$.

\begin{remark}
\label{remark:discarded_region}
Lemma~\ref{lemma:discarded_region} tells us that the region discarded by Algorithm~1 under the pool-based or stream-based setting, will have $P_X$ measure smaller than $1/n$ with probability at least $1-1/n$.
For the remaining part of the input space, i.e, $\X \setminus \tilde{\X}_n$, all the three active learning frameworks are equivalent because in all the three frameworks we can request a label from a point in any cell in the region $\X \setminus \tilde{\X}_n$. 
\end{remark}

\subsection{Proof of Theorem~\ref{theorem:setting1}}
\label{appendix:setting1-proofs}

We begin with a lemma which gives us  high probability upper and  lower bounds on the estimates of the regression function values at the active points.

\begin{lemma}
\label{lemma:step_1_1}
The event $\Omega_1=\cap_{t\geq1}\Omega_{1,t}$ occurs with probability at least $1-\frac{1}{n}$, where the events $\Omega_{1,t}$, for $t\geq 1$, are defined as 
\begin{equation*}
    \Omega_{1,t} \coloneqq \big\{ |\hat{\eta}(x_{h,i}) - \eta(x_{h,i})| \leq e_t(n_{h,i}), \ \forall x_{h,i} \in \X_t\big\}, \quad\; \text{with } \;\; e_t(n_{h,i}) \coloneqq \sqrt{\frac{2 \log(2\pi^2t^3n/3)}{n_{h,i}(t)}},
\end{equation*}
where $n_{h,i}(t)$ is the number of times that $x_{h,i}$ has been queried up until time $t$. 
\end{lemma}
\begin{proof}
It suffices to show that $P(\Omega_{1,t}^c) \leq \frac{6}{n\pi^2 t^2}$. The result then follows from a union bound over all $t \geq 1$ and the fact that $\sum_{t\geq 1}\frac{1}{t^2} = \frac{\pi^2}{6}$. 
Now, for a given $x_{h,i} \in \X_t$ and for any $e_t(n_{h,i}(t))>0$, by Hoeffding's inequality, we have 
\begin{equation*}
    P\big(|\hat{\eta}(x) - \eta (\xhi)| > e_t(n_{h,i}(t))\big) \leq 2e^{-ne_t(n_{h,i}(t))^2/2}.
\end{equation*}
Finally, by selecting $e_t(n_{h,i}(t)) = \sqrt{\frac{2 \log\big((2\pi^2t^3n)/3\big)}{n_{h,i}(t)}}$, we obtain 
\begin{equation*}
    P(\Omega_{1,t}^c) \leq 2\sum_{(h,i): x_{h,i} \in \X_t}e^{-n_{h,i}(t)a_{h,i}^2/2} \leq \sum_{(h,i): x_{h,i} \in \X_t} \frac{3}{n \pi^2 t^3} \stackrel{\text{(a)}}{\leq} \frac{6}{n \pi^2 t^2}.
\end{equation*}
{\bf (a)} follows from the fact that $|\X_t| \leq 2t$, for all $t \geq 1$.
This is because of the following reasoning: $|\X_0| = 1$, and for any $1\leq i \leq t$, we must have $|\X_i| \in \{ |\X_{i-1}| + 1, |\X_{i-1}|\} \leq |\X_{i-1}| + 1 $. Thus by induction, we get $|\X_t| \leq t+ 1$, which is no larger than $2t$, for $t \geq 1$. 
\end{proof}

We now present a result on the monotonicity of the term $\It{1}\lp \xhit \rp$ which will be used in obtaining bounds on the estimation error of the regression function. 
\begin{lemma}
\label{lemma:step1_4}
$\It{1}\lp \xhit \rp$ is non-increasing in $t$. 
\end{lemma}

\begin{proof}
The proof of this statement relies on the monotonic nature of $u_t(\xhi)$ and $l_t(\xhi)$. More specifically, for any $\xhi \in \Xtu$, we have $I_{t+1}^{(1)}(\xhi) \leq \It{1}(\xhi)$ due to the definition of $u_t(\xhi)$ and $l_t(\xhi$ given in Step~2 of Algorithm~1. Furthermore, if the algorithm refines the cell $\X_{h_t,i_t}$, then by definition, we also have $I_{t+1}^{(1)}(\xhi) \leq \It{1}\lp \xhit\rp$, for $h=h_{t}+1$ and $i \in \{2i_t-1, 2i_t\}$, due to the cell refinement rule. These two statements together imply that  the term $\sup_{\xhi \in \Xtu}\It{1}(\xhi)$ is also a non-increasing term.
\end{proof}

We next derive a bound on the error in estimating the regression function at the cells close to the threshold values $1/2- \gd$ and $1/2 + \gd$. 
\begin{lemma}
\label{lemma:step1_2}
Suppose $t_n$ is the time at which Algorithm~1 stops (i.e.,~performs the $n^{th}$ query) and $\X_{t_n}^{(u)}$ is the set of unclassified points at time $t_n$. Define the term $\tilde{D} = \max\{\tilde{D}_1, \tilde{D}_2\}$, where $\{\tilde{D}_j\}_{j=1}^2 \coloneqq D_{1/2 + (-1)^j\gd}\lp \zeta_1 \rp$ in which $\zeta_1(r) = 3L(v_1/(v_2\rho))^{\beta}r^{\beta}$ and $D_{\lambda}(\zeta)$ is from Definition~\ref{def:dimension1}. Then for large enough $n$ and for any $a>\tilde{D}$, with probability at least $1-\frac{1}{n}$, we have  
\begin{equation*}
    |\eta(\xhi) - \hat{\eta}(\xhi)| \leq b_n = \frac{3Lv_1^{\beta}}{\rho^\beta} \lp\frac{2C_a}{L^2v_1^{2\beta}v_2^a} \rp^\beta \lp\frac{\log(2\pi n/3)}{n}\rp^{\frac{\beta}{(a+2\beta)}}, \quad \text{for all } \; \xhi \in \X_{t_n}^{(u)}. 
\end{equation*}
\end{lemma}
\begin{proof}
First note that the algorithm refines the cell associated with a point $\xhi$, if $2e_t(n_{h,i}(t)) \leq V_h = L(v_1\rho^h)^{\beta}$. The uncertainty of the estimate of $\eta(\xhi)$ can be further upper-bounded at any time $t$ by setting $t=1$ in the expression of $e_t(n_{h,i}(t))$, i.e.,
\begin{align*}
   2e_t\lp n_{h,i}(t) \rp \leq \sqrt{\frac{8\log(2 \pi^2 n/3)}{n_{h,i}(t)}}. 
\end{align*}
Thus, to find an upper-bound on the number of times a point $x_{h,i}$ is queried by the algorithm, it suffices to find the number of queries sufficient to ensure that $\sqrt{(8\log(2 \pi^2 n/3))/n_{h,i}(t)}$ is less than or equal to $V_h$. Equating this term with $V_h$, we obtain
\begin{equation}
\label{eq:step1_proof1}
n_{h,i}(t_n) \leq \frac{8 \log (2 \pi^2 n/3)}{L^2v_1^{2\beta}\rho^{2h\beta}},
\end{equation}
where $t_n$ is the time at which the budget of $n$ label queries is exhausted and the algorithm stops. Now, by definition, a point $x_{h,i}$ belongs to the set $\Xtu$, only if $\{1/2-\gd, 1/2+\gd\} \cap [l_t(\xhi), u_t(\xhi)]$ $\neq \emptyset$. Suppose for a given $\xhi \in \X_t$, the interval $[l_t(\xhi), u_t(\xhi)]$ contains $1/2 - \gd$. This implies that for $h \geq 1$, we have
\begin{align*}
    \sup_{x \in \X_{h,i}} |\eta(x)- 1/2 + \gd| & \leq \max \{ u_t(\xhi) + V_h - 1/2 + \gd, \ 1/2 - \gd - l_t(\xhi) - V_h \} \\
    & \stackrel{\text{(a)}}{\leq}  u_t(\xhi) - l_t(\xhi) +  \\
    & \stackrel{\text{(b)}}{\leq} V_{h-1} \leq  3L\lp v_1 \rho^{h-1}\rp^{\beta}. 
\end{align*}
{\bf (a)} follows from the condition that $l_t(\xhi)  \leq  1/2 - \gd \leq u_t(\xhi) $. \\
{\bf (b)} follows from the rule used for refining the parent cell of $\xhi$, after which $\xhi$ becomes active. \\

Now, we define the function $\zeta_1(r)=3L(v_1/(v_2\rho))^{\beta}r^\beta$ and use it to define the term $\tilde{D}_1 = D_{1/2 - \gd}(\zeta_1)$ (see Definition~\ref{def:dimension1}). Similarly, we define $\tilde{D}_2 = D_{1/2 + \gd}(\zeta_1)$ at the other threshold value and introduce the notation $\tilde{D} = \max\{\tilde{D}_1, \tilde{D}_2\}$. Thus, the total number of points that are activated by the algorithm at level $h$ of the tree, denoted by $N_h$, can be upper-bounded by the packing number of the set $\X_{1/2-\gd}\lp \zeta_1(v_2\rho^h)\rp\cup \X_{1/2 + \gd}\lp \zeta_1(v_2 \rho^{h}) \rp$ with balls of radius $v_2\rho^h$. Now, by the definition of $\tilde{D}$, for any $a>\tilde{D}$, there exists a $C_a < \infty$ such that we can upper-bound $N_h$ with the term $2C_a(v_2\rho^h)^a$. Using the bound on $N_h$ and $n_{h,i}(t_n)$, we observe that the number of queries made by the algorithm at level $h$ of the tree is no more than $N_h n_{h,i}(t_n)$. Hence, for any $H \geq 1$, we have 
\begin{align}
    \sum_{h=0}^{H}N_h n_{h,i}(t_n) &\leq \frac{8\log(2\pi^2 n/3) C_av_2^{-a}}{L^2v_1^{2\beta}}\sum_{h=0}^H\lp \frac{1}{\rho}\rp^{h(a + 2\beta)} \nonumber \\
    &\leq   \frac{8\log(2\pi^2 n/3) C_av_2^{-a}}{L^2v_1^{2\beta}} \lp \frac{1}{\rho}\rp^{H(a + 2\beta)}. \label{eq:temp1} 
\end{align}
Next, we need to find a lower-bound on the depth in the tree that has been explored by the algorithm. This can be done by finding the largest $H$ for which~\eqref{eq:temp1} is smaller than or equal to $n$. By equating~\eqref{eq:temp1} with $n$, we obtain the following relation for the largest such value of $H$, denoted by $H_0$,
\begin{equation}
\label{eq:temp2}
    \lp \frac{1}{\rho}\rp^{H_0} = \lp \frac{L^2 v_1^{2\beta} v_2^a}{8C_a}\rp^{1/(a+2\beta)}\lp \frac{n}{\log(1\pi^2 n/3)}\rp^{1/(a + 2\beta)}. 
\end{equation}
Now, for any $x \in \cup_{\xhi \in \X_{t_n}^{(u)}}\X_{h,i}$, we must have
\begin{equation*}
    |\hat{\eta}(x) - \eta(x)| = |\hat{\eta}_{t_n}(\pi_{t_n}(x)) - \eta(x)| \leq u_{t_n}(x) - l_{t_n}(x) \stackrel{\text{(a)}}{ \leq} I_{t_n}^{(1)}(x_{h_{t_n}, i_{t_n}}). 
\end{equation*}
{\bf (a)} follows from the point selection rule of the algorithm. \\

Lemma~\ref{lemma:step1_4} implies that if the algorithm is evaluated a point at level $H_0$ at some time $t\leq t_n$, then we have
\begin{equation*}
\sup_{\xhi \in \X_{t_n}^{(u)}}I_{t_n}^{(1)}(\xhi) \leq 3V_{H_0-1} = 3L (v_1\rho^{H_0-1})^{\beta} \coloneqq b_n, 
\end{equation*}
where 
\begin{align*}
b_n = \frac{3L v_1^{\beta}}{\rho^\beta}\lp \frac{8C_a}{L^2v_1^{2\beta}v_2^a} \rp^{\beta/(a+2\beta)} \lp \frac{\log(2\pi^2 n/3)}{n}\rp^{\beta/(a+2\beta)} = \mathcal{O}\lp \lp \frac{n}{\log n}\rp^{-\beta/(a+2\beta)}\rp. 
\end{align*}
\end{proof}

\noindent
Finally, we combine Lemma~\ref{lemma:step1_2} with the margin assumptions to obtain the required result.

\begin{lemma}
\label{lemma:step1_4}
The excess risk of the classifier $\hat{g}$ in~\eqref{eq:classifier-algo1}, learned by Algorithm~1, w.r.t.~the optimal classifier in the fixed cost of abstention setting, with the fixed abstention cost $\lambda = 1/2 - \gd$, satisfies $R_\lambda(\hat{g})-R_\lambda(g^*_\lambda) \leq \tilde{\mathcal{O}}\lp n^{-\beta(\alpha_0+1)/(2\beta+a)}\rp$. 
\end{lemma}

\begin{proof}
By definition of the classifier $\hat{g} = (\hat{G}_0, \hat{G}_1, \hat{G}_\Delta)$, under the event $\Omega_1$, the set $\hat{G}_\Delta \subset \Gds$.

Now, by Lemma~\ref{lemma:step1_2}, we know that $\sup_{\xhi \in \X_{t_n}^{(u)}} \It{1}(\xhi) \leq b_n$, which for $n$ large enough ensures that $b_n \leq \gd$ leading to $\hat{G}_0 \subset \{x \in \X\ \mid \ \eta(x) \geq 1/2\}$.  This  implies that $\hat{G}_0 \cap G_1^* = \emptyset$. Similarly, we can obtain $\hat{G}_1 \cap G_0^* = \emptyset$. Thus, the excess risk of the estimated classifier can be written as %
\begin{align*}
R_\lambda\lp \hat{g} \rp - R_\lambda\lp g^*_\lambda \rp &= \int_{\hat{G}_0}\eta(x) dP_X + \int_{\hat{G}_1}\big(1-\eta(x)\big)dP_X + \lambda P_X\lp \hat{G}_\Delta \rp \\ 
&\quad - \int_{G_0^*}\eta(x) dP_X - \int_{G_1^*}\big(1-\eta(x)\big)dP_X - \lambda P_X(\Gds) \\
&= \int_{\hat{G}_0 \cap G^*_\Delta}(\eta(x) - \lambda)dP_X + \int_{\hat{G}_1\cap G^*_\Delta}(1-\lambda - \eta(x))dP_X \\
&\quad + \int_{\hat{G}_\Delta \cap G^*_0}(\lambda - \eta(x))dP_X + \int_{\hat{G}_\Delta \cap G^*_\Delta} (\eta(x) - 1 + \lambda)dP_X \\
& \leq  b_nP_X\big(|\eta(X)-\lambda| \leq b_n\big) + b_n P_X\big(|\eta(X)-1+\lambda|\leq b_n\big) \\
& \leq 2C_0b_n^{1 + \alpha_0}. 
\end{align*}
\end{proof}


\newpage
\section{Algorithm for Setting~2: Bounded rate with  $P_X$ known}
\label{appendix:setting2}

\subsection{Details of Algorithm~2}
\label{appendix:setting2-details}

\paragraph{Outline of Algorithm~2.}
Similar to Algorithm~1, at any time $t$, Algorithm~2 maintains a set of active points $\X_t$, which is further partitioned into unclassified $\Xtu$, classified $\Xtc$, and discarded $\Xtd$ sets.
We sort the points in $\Xtu$ in terms of how far an estimate of the regression function at each point is away from $1/2$, and use these points/cells\footnote{Since there is a one-to-one mapping between a point and the cell associated with it, we use the terms point and cell interchangeably throughout the paper.} along with the marginal $P_X$ to obtain upper and lower bounds on the true threshold $\gd$.
We define the set of unclassified points based on the estimated threshold and the estimation error.
These estimates of the threshold are used while updating the unclassified active set $\Xtu$. The point selection and cell refinement rules are the same as in Algorithm~1. 


\paragraph{Steps of Algorithm~2.}
The algorithm proceeds in the following steps:
\begin{enumerate}[wide, labelwidth=!, labelindent=0pt]

\item At $t=0$, set $\X_0 = \{x_{0,0}\}$, $\X_0^{(u)}=\X_0$, $\X_0^{(c)} =\X_0^{(d)} = \{\}$,  $u_0(x_{0,0}) = +\infty$, $l_0(x_{0,0}) = -\infty$. 

\item For $t\geq 1$, calculate the upper-bound $u_t(\xhi)$ and the lower-bound $l_t(\xhi)$ for every $\xhi \in \X_t$, as it was done in Algorithm~1. 

\item Define the piecewise constant function $f_t(\cdot)$ as
\begin{equation*}
f_t(x) = 
\begin{cases} 
u_t(\pi_t(x)) & \text{if } \; u_t(\pi_t(x))<1/2, \\
l_t(\pi_t(x)) & \text{if } \; l_t(\pi_t(x)) >1/2, \\
1/2 & \text{otherwise},  
\end{cases}
\end{equation*}
where $\pi_t(\cdot)$ is defined by~\eqref{eq:temp000}. Note that by construction, we have $|f_t(x)-1/2| \leq |\eta(x)-1/2|,\;\forall x\in\X$, a property that will play an important role in the analysis of the algorithm.

\item Sort the points/cells in $\X_t \setminus \Xtd$ in ascending order of their $|f_t(\cdot)-1/2|$ value. We denote the ordered cells by $\Et{j}$ and their corresponding (ordered) center points by $\xtj{j}$. We now introduce the term $k_t \coloneqq \min \big\{ k \geq 1 \ \mid \ P_X\big(\cup_{j=1}^k\Et{j}\big) > \delta \big\}$ and use it to define the terms $\hgone = f_t\big( \xtj{k_t-1}\big)$, $\hgtwo = f_t\big(\xtj{k_t}\big)$, $S_1 = \cup_{j=1}^{k_t-1}\Et{j}$, and $S_2 = \cup_{j=1}^{k_t}\Et{j}$. 

\item Select a candidate point $\xhit$ as in Algorithm~1 and introduce the notation $J_t = \It{1}\lp \xhit \rp = \max_{\xhi \in \Xtu} \It{1}\lp \xhi \rp$. 

\item Refine the cell or request a label as in Step~5 of Algorithm~1. 

\item The set of unclassified points $\Xtu$ is updated at the end of round $t$ as
\begin{equation}
\label{eq:update_rule_alg2}
\begin{split}
    \Xtu \leftarrow \bigg\{ \xhi \in \Xtu\, \mid &\, [l_t(x_{h,i}), u_t(x_{h,i})] \bigcap \bigg( [1/2 +\hgone, 1/2 + \hgtwo + 3J_t ] \bigcup  \\
    &[1/2- \hgtwo -3J_t, 1/2-\hgone] \bigg) \neq \emptyset \bigg\}. 
    \end{split}
\end{equation}

\item If the algorithm stops at time $t_n$, the final estimate of the regression function is calculated as in Step~6 of Algorithm~1. 

\item Similar to Step~4 above, sort all the cells of $\X_{t_n}\setminus \X_{t_n}^{(d)}$ in terms of $|f_{t_n}-1/2|$ value (and denote them by $E'_{(j)}$).
Define $k'$ as follows:
\begin{equation}
    \label{eq:algo2_temp1}
k' \coloneqq \max\{k \geq 1 \mid P_X\lp \cup_{j=1}^{k}E'_{(j)} \rp \leq \delta \}.
\end{equation}
For $n$ large enough $E'_{(k'+1)}$ will be completely contained in either $\{x \in \X \mid \eta(x)-1/2 \leq 0\}$ or in $\{x \in \X \mid \eta(x)-1/2 \geq 0\}$. Introduce a variable $j'$ and assign to it the  value $0$  if it is the former. Otherwise set $j'=1$. Finally, define $c' = \lp \delta - P_X\lp \cup_{j=1}^{k'}E'_{(j)}\rp \rp/P_X\lp E'_{(k'+1)}\rp$.

\item Finally, the (possibly randomized) classifier returned by the algorithm is defined as 
\begin{equation}
\label{eq:classifier-algo2}
    \hat{g}(x) = 
    \begin{cases}
        \Delta & \text{if } \;\; x \in \cup_{j=1}^{k'}E'_{(j)} \\
        \big((1-c')(1-j'), (1-c')j', c'\big) & \text{if } \;\; x \in E_{(k'+1)}' \\
        1 & \text{if } \;\; u_{t_n}\lp \pi_{t_n}(x) \rp > 1/2  \text{ or } x \in \tilde{\X}_n \\
          0 & \text{if } \;\; l_{t_n}\big(\pi_{t_n}(x)\big) < 1/2  \text{ and } x \not \in \tilde{\X}_n, \\
    \end{cases}
\end{equation}
where $\pi_{t_n}(\cdot)$ is the projection onto $\X_{t_n}$ as defined by~\eqref{eq:temp000}.
\end{enumerate}

\begin{remark}
\label{remark:algo2_remark1}
The key difference between Algorithm~1 and Algorithm~2 lies in the rule used for updating the set of unclassified points. In Algorithm~1, this update was straightforward as the threshold was assumed to be known. In Algorithm~2, we need to use the current estimate of the regression function to obtain upper and lower bounds on the true threshold, and then use these bounds to decide which parts of the inputs space have to be further explored, i.e.,~remain unclassified. The quantity $f_t(\cdot)$ introduced in Step~3 has the property that $|f_t(\cdot) - 1/2|$ is a lower-bound on $|\eta(\cdot) - 1/2|$. This property is useful for obtaining the confidence bounds for the estimated thresholds. 
\end{remark}

\subsection{Pseudo-code of Algorithm~2}
\label{appendix:setting2-algo}

In this section, we report the pseudo-code of Algorithm~2 that was outlined and described in Section~\ref{subsec:setting2}. This is our active learning algorithm for the setting in which the learner does not have the knowledge of the true threshold value, but has access to the true marginal $P_X$ . This setting is equivalent to the assumption of having infinite unlabelled samples. \\

\begin{algorithm}
\SetAlgoLined
\SetKwInput{Input}{Input}
\SetKwInput{Output}{Output}
\SetKw{Init}{Initialize}

\caption{Active learning algorithm for the known $P_X$ setting.}
\Input{$n$,$\gd$, $L$, $\beta$, $v_1$, $v_2$, $\rho$} 
 
\Init{$t=1$, $n_e=0$, $\X_t = \{x_{0,1}\}$, $\Xtu = \X_t$, $\Xtc= \emptyset$}
 
\If{$n_e = 0$}
{
$u_t(x_{0,1}) = + \infty$ \\
$l_t(x_{0,1}) = - \infty$ 
}
\vspace{0.5em}

\While{$n_e \leq n$}
{
\vspace{0.5em}
\tcc{Update the terms $u_t$, $l_t$, and $f_t$}

\For{$x_{h,i}\in \Xtu$} 
 {
 $u_t(\xhi) \leftarrow \min \big\{ \bar{u}_t(\xhi), u_{t-1}(\xhi) \big\}; \qquad\qquad l_t(\xhi) \leftarrow \max \big\{ \bar{l}_t(\xhi), l_{t-1}(\xhi) \big\}$ \\
 $f_t \leftarrow \max\big\{0,\, l_t(x) -1/2,\, 1/2-u_t(x)\big\}$
 }

define $k_t$, $\hgone$, $S_1$, $\hgtwo$, and $S_2$ \\

 

 
\tcc{Choose a candidate point with most uncertainty} 
 
$x_{h_t,i_t} \in \argmax_{x_{h,i}\in \Xtu} \It{1}(\xhi) = u_t(\xhi) - l_t(\xhi)$

\vspace{0.5em}
 
\tcc{ Refine or Label}
 
\eIf{$\;e_t\big(n_{h,i}(t)\big)< L(v_1\rho^{h_t})^{\beta}$}{
$\Xtu \leftarrow \Xtu \setminus \{\xhit\}\cup \{ x_{h_t+1, 2i_t-1}, x_{h_t+1, 2i_t}\}$ \\
$u_t(x_{h_t+1,2i_t-1})\leftarrow u_t(\xhit);\qquad$  $l_t(x_{h_t+1,2i_t-1})\leftarrow l_t(\xhit)$ \\
$u_t(x_{h_t+1,2i_t})\leftarrow u_t(\xhit);\qquad\;\;\;\;$ $l_t(x_{h_t+1,2i_t})\leftarrow l_t(\xhit)$ 
}
{
call REQUEST\_LABEL
}
 
\vspace{0.5em}

\tcc{Update $\Xtu$ and $\Xtc$}
$\mathcal{Z}_t \leftarrow \left\{ \xhi \in \Xtu\, \mid \, |f_t(\xhi) - 1/2| > \hgtwo + 3\It{1}(\xhit) \quad\text{OR}\quad |f_t(\xhi) - 1/2| < \hgone \right\}$ \\
$\Xtu \leftarrow \Xtu \setminus \mathcal{Z}_t$ \\
$\Xtc \leftarrow \Xtc \cup \mathcal{Z}_t$ \\

$t \leftarrow t+1$
}
\Output{$\hat{g}$ defined by Eq.~\ref{eq:classifier-algo2}\\}
\label{alg:basic_alg2}
\end{algorithm}

\subsection{Analysis of Algorithm~2}
\label{appendix:setting2-proofs}
We begin this section by stating the main result, which provides high probability upper bound on the excess risk of Algorithm~2. 

\begin{theorem}
\label{theorem:setting2}
Let the assumptions~(MA) and~(H\"O) hold with parameters $C_0>0$, $\alpha_0 \geq 0$, $L >0$, and $0<\beta\leq 1$. Let $\tilde{D}$ represent the dimension term introduced in Remark~\ref{remark:dimension}. 
Moreover, assume that the regression function $\eta(\cdot)$ is such that $|\eta(X)-1/2|$ has no atoms. Then, for $n$ large enough, the following statements are true for the classifier $\hat{g}$ defined by~\eqref{eq:classifier-algo2}, with probability at least $1-\frac{2}{n}$:
\begin{enumerate}[wide, labelwidth=!, labelindent=0pt]
\item The classifier $\hat{g}$ is feasible for~\eqref{bayes_optimal_delta}, i.e.,~$P_X(\hat{G}_\Delta) \leq \delta$. 
\item For any $a> \tilde{D}$, the excess risk of the classifier $\hat{g}$ satisfies 
\begin{equation}
\label{eq:excess-risk-algo2-2}
    R(\hat{g}) - R(g^*_{\delta}) \leq 2C_0 \lp 1/2 - \gd + J_{t_n} \rp J_{t_n}^{\alpha_0}, \qquad\;\; \text{where} \quad J_{t_n} = \tilde{\mathcal{O}}\big( n^{-\beta/(2\beta + a)} \big).
\end{equation}
%
%
%
The hidden constant in~\eqref{eq:excess-risk-algo2-2} depends on the parameters $L$, $v_1$, $v_2$, $\rho$, $\beta$, and $a$. 
\end{enumerate}
\end{theorem}

\paragraph{Proof Outline.}
The proof of  Theorem~\ref{theorem:setting3} follows the same general outline as the proof of Theorem~\ref{theorem:setting1}. The main new task is to establish that the estimated thresholds, $\hgone$ and $\hgtwo$, are close enough to the true threshold $\gd$. These results are proved in Lemmas~\ref{lemma:algo2_lemma1}, \ref{lemma:algo2_lemma2} and \ref{lemma:algo2_lemma3}, resulting in the equations~\eqref{eq:gamma1} and \eqref{eq:gamma2}. We then obtain the estimation error on the regression function in Lemma~\ref{lemma:algo2_lemma4}. Finally, to complete the proof we obtain a bound on the excess risk in terms of the regression function estimation error and employ the margin condition.

In this section, we will work under the assumption that the event defined in Lemma~\ref{lemma:discarded_region} as well as the event  $\Omega_1$ defined in Lemma~\ref{lemma:step_1_1} hold.

The probability of both of these events occurring simultaneously is at least $1-2/n$. 

We first present a set of results that tell us how close the estimated thresholds $\hgone$ and $\hgtwo$ defined in Step~4 of Algorithm~2 are to the true threshold value $\gd$. 

\begin{lemma}
\label{lemma:algo2_lemma1}
Assume that the random variable $|\eta(X)-1/2|$ has no atoms. Then under the event $\Omega_1$, we have $\hgone \leq \gd$. 
\end{lemma}
\begin{proof}
If $\hgone = 0$, then the result follows trivially since $\gd\geq 0$. For the case that $\hgone >0$, we proceed as follows:
\begin{align*}
P_X\lp |\eta - 1/2| \leq \gd \rp = \delta &\stackrel{\text{(a)}}{\geq}  P_X\lp S_1 \rp \geq P_X\lp |f_t-1/2| <\hgone \rp \\
& \stackrel{\text{(b)}}{\geq} P_X\lp |\eta -1/2| <\hgone\rp \stackrel{\text{(c)}}{=} P_X\lp |\eta-1/2|\leq \hgone \rp.
\end{align*}
{\bf (a)} follows from the definition of the set $S_1$ in Step~4 of Algorithm~2. \\
{\bf (b)} follows from the fact that $|f_t(\cdot)-1/2|$ is a lower-bound of the function $|\eta(\cdot)-1/2|$. \\
{\bf (c)} follows from the assumption that the random variable $|\eta(X)-1/2|$ has no atoms. 
\end{proof}

\begin{lemma}
\label{lemma:algo2_lemma2}
Under the assumptions of Lemma~\ref{lemma:algo2_lemma1}, we have $\hgtwo + J_t \geq \gd$. 
\end{lemma}

\begin{proof}
We observe that for any $x \in S_2$, we must have $\eta(x) \leq f_t(x) + \It{1}\big(\pi_t(x)\big) \leq f_t(x) + J_t$. Since $|f_t(x)-1/2| \leq \hgtwo$, for all $x \in S_2$, we have $S_2 \subset \big\{x \in \X \, \mid \, |\eta(x)-1/2| \leq J_t + \hgtwo\big\}$, and thus, we may write
\begin{align*}
    P_X\lp |\eta-1/2|\leq \gd\rp = \delta  < P_X\lp S_2 \rp 
      \leq P_X\lp |\eta-1/2| \leq \hgtwo + J_t \rp. 
\end{align*}
This implies that $\gd \leq \hgtwo + J_t$ and proves the lemma. 
\end{proof}

\begin{lemma}
\label{lemma:algo2_lemma3}
Under the assumptions of Lemma~\ref{lemma:algo2_lemma1}, we have $\hgtwo \leq \hgone +  J_t$. 
\end{lemma}

\begin{proof}
We first need to show that because of the rule used for updating the unclassified points, there must exist a point $\xhi \in \Xtu\setminus S_1$ such that $\bar{\X}_{h,i}\cap \bar{S}_1 \neq \emptyset$,
where we use $\bar{A}$ to denote the closure of any set $A$ as a subset of the metric space $(\X, d)$.
To obtain this result, we proceed by contradiction. Suppose this is not true.
Since the cells associated with points in $\X_t$ at any time $t$ partition the entire space $\X$, there must exist a point $\xhi \in \Xtc \setminus S_1$ that shares a boundary point with $S_1$, i.e.,~there exists a $1\leq j \leq k_t-1$ such that $\bar{E}_{(j)}^{(t)} \cap \bar{\X}_{h,i} \neq \emptyset$ for $x_{h,i} \in \Xtc \setminus S_1$. Let $x$ denote a point in $\bar{E}_{(j)}^{(t)} \cap \bar{\X}_{h,i}$.

Now suppose $\xhi$ was classified at some time $t_0 < t$. Then, by the rule used for updating the set $\Xtu$ and by the definition of $S_1$, we have $|f_{t_0}(\xhi) -1/2| \geq \hat{\gamma}^{(t_0)}_2 + 3J_{t_0} \geq \gd + 2J_{t_0} \geq \gd + 2J_{t}$, where the second inequality is from Lemma~\ref{lemma:algo2_lemma2} and the third inequality uses the fact that $J_t$ is non-increasing in $t$ (this can be obtained similar to the proof of Lemma~\ref{lemma:step1_4} used in the analysis of Algorithm~1).
Furthermore, because of the {\em minimum} in the definition of $u_t(\cdot)$, we have $u_{t}(\xhi) \leq u_{t_0}(\xhi)$.
Similarly, we have $l_t(\xhi) \geq l_{t_0}(\xhi)$.
Together these two results imply that $|f_t(\xhi) - 1/2| \geq |f_{t_0}(\xhi) - 1/2| \geq \gd + 2J_t$. Since the point $x$ lies in $\bar{E}_{(j)}^{(t)}$, we may write
\begin{equation}
\label{eq:appendixB1}
|\eta(x)-1/2| \leq |f_t\lp \xtj{j}\rp - 1/2| + J_t \leq \hgone  + J_t \leq \gd + J_t.
\end{equation}
Also, since $x \in \bar{\X}_{h,i}$, we have 
\begin{equation}
\label{eq:appendixB2}
|\eta(x)-1/2| \geq |f_t(\xhi) - 1/2| \geq \gd + 2J_t. 
\end{equation}
Together \eqref{eq:appendixB1} and \eqref{eq:appendixB2} imply that
\begin{equation*}
\gd + 2J_t \leq |\eta(x) - 1/2| \leq \gd + J_t,
\end{equation*}
which gives us the required contradiction, since $J_t>0$. Thus, there must exist a point $\xhi \in \Xtu$ that shares a boundary point with $S_1$. Now, we use this fact to complete the proof as follows:
\begin{align*}
    \hgtwo &\stackrel{\text{(a)}}{\leq} |f_t(\xhi) - 1/2|  \stackrel{\text{(b)}}{\leq} |\eta(x) - 1/2| \stackrel{\text{(c)}}{\leq} |f_t\lp \xtj{j}\rp - 1/2| + \It{1}\lp \xtj{j}\rp \\
    & \stackrel{\text{(d)}}{\leq} |f_t \lp \xtj{j} \rp - 1/2 | + J_t \stackrel{\text{(e)}}{\leq} \hgone + J_t. 
\end{align*}
{\bf (a)} follows from the fact that when the cells in $\Xtu$ are sorted in the increasing order of $|f-1/2|$ value, the position of $\Xhi$ must be at least $k_t$. \\
{\bf (b)} can be obtained as follows: Fix an $\epsilon >0$. By the continuity of $\eta$, there exists an $\epsilon_1>0$ such that if $d(x,z) < \epsilon_1$, then $|\eta(x) - \eta(z)| \leq \epsilon$. Since $x \in \bar{\X}_{h,i}$, for every $\epsilon_1 >0$, there exists a $z \in \Xhi$ with $d(z,x)< \epsilon_1$. Furthermore, from the definition of $f_t$, we have $|\eta(z)-1/2| \geq |f_t(\xhi)-1/2|$. Combining these, we obtain that $|\eta(x) - 1/2| + \epsilon \geq |\eta(z)-1/2| \geq |f_t(\xhi) - 1/2|$. Since $\epsilon>0$ was arbitrary, we obtain the required result. \\
{\bf (c)} follows from similar reasoning as in {\bf (b)}\\
{\bf (d)} follows from the definition of $J_t$ as the largest $\It{1}\lp \xhi \rp$ value over points in $\Xhi$. \\
{\bf (e)} uses the fact that $j \leq k_t-1$. 
%
%
\end{proof}

From the previous lemmas, we can reach the following conclusion:
\begin{align}
    -2J_t \leq \hgone - \gd &\leq 0, \label{eq:gamma1} \\
    -J_t \leq \hgtwo - \gd & \leq J_t. \label{eq:gamma2} 
\end{align}
Our next result gives us an upper-bound on the value of $J_{t_n}$. 
\begin{lemma}
\label{lemma:algo2_lemma4}
If Algorithm~2 stops at time $t_n$, for any $a>\tilde{D}$, where $\tilde{D}$ is defined in Remark~\ref{remark:dimension}, we have
\begin{equation*}
J_{t_n} = \mathcal{O}\lp \lp n/\log n\rp^{-\beta/(a+2\beta)}\rp.
\end{equation*}
\end{lemma}
\begin{proof}
The proof of this statement follows the steps similar to that used in obtaining the bound on $b_n$ in Lemma~\ref{lemma:step1_2} in Appendix~\ref{appendix:setting1-proofs}. Since the rule used for refining is the same as in Algorithm~1, the same bound on $n_{h,i}(t_n)$ holds in this case as well.  

Now, because of the rule used for updating the set $\Xtu$, we know that if $x \in \cup_{\xhi \in \Xtu}\Xhi$, then we have 
\begin{align*}
    |\eta(x) - 1/2| & \in [\hgone - J_t, \hgtwo + 4J_t ],   
\end{align*}
which on using \eqref{eq:gamma1} and \eqref{eq:gamma2} implies that 
\begin{align*}
    |\eta(x) -\lambda| & \leq 5J_t \quad \text{for} \quad \lambda \in \{1/2 - \gd, 1/2 + \gd\}. 
\end{align*}
%
%

Thus, the set $\cup_{\xhi \in \Xtu}\Xhi \subset \big\{x \in \X \ \mid \ |\eta(x)-\lambda| \leq  5J_t\big\}$ for $\lambda \in \{1/2-\gd, 1/2 + \gd\}$.
Finally, if the algorithm evaluates a point at level $h \geq 1$ of the tree of partitions at time $t$, then we must have $J_t \leq 2V_{h-1}$.
This follows from the cell refinement rule.
Combining these, we obtain that the set of points evaluated by the algorithm at level $h$ of the tree  must lie in the set $\{ \xhi \in \X_h \ \mid \ |\eta(\xhi) -\lambda| \leq  10V_{h-1}\}$.
The rest of the proof uses the same arguments as those used for bounding $b_n$ in Lemma~\ref{lemma:step1_2} in Appendix~\ref{appendix:setting1-proofs} and is omitted here. 
\end{proof}

Now we are ready for the final result, i.e.,~to find an upper-bound on the excess risk of the classifier returned by Algorithm~2. 
\begin{lemma}
\label{lemma:algo2_lemma5}
For $n$ large enough to ensure that $2J_t< \gd$, we have 
\[
R(\hat{g}) - R(g^*_\delta) \leq 2C_0 J_{t_n}^{1+\alpha_0}. 
\]
\end{lemma}

\begin{proof}
From the definition of the classifier, and the assumption that $n$ is large enough to ensure that $2J_{t_n}< \gd$, we can again show that $\hat{G}_j \cap G^*_{1-j}= \emptyset$ for $j=0,1$. 
%
%
We then have
\begin{align*}
    R(\hat{g}) - R(g^*_\delta) &= \int_{\hat{G}_0}\eta dP_X + \int_{\hat{G}_1 }(1-\eta)dP_X - \int_{G_0^*}\eta dP_X - \int_{G_1^*}(1-\eta)dP_X  \\
    &\stackrel{(a)}{=} \int_{\hat{G}_0\cap \Gds}(\eta - \lambda)dP_X + \int_{\hat{G}_1 \cap \Gds}(1-\lambda -\eta)dP_X + \int_{\hat{G}_\Delta \cap G^*_0}(\lambda - \eta)dP_X\\ 
    & \qquad + \int_{\hat{G}_\Delta \cap G^*_1}(\eta - 1 + \lambda)dP_X \\
    &\stackrel{(b)}{\leq}(2J_{t_n}) 2C_0  (2J_{t_n})^{\alpha_0} = \mc{O}\lp J_{t_n}\rp^{1 +\alpha_0}.
\end{align*}
In the above display, \\
{\bf (a)} follows from the fact that $\hat{G}_j \cap G^*_{1-j} = \emptyset$ for $j=0,1$, and by adding $\lambda P_X\lp \hat{G}_\Delta\rp$ and subtracting $\lambda P_X(\Gds)$, and using the fact that $P_X(\Gds) = P_X(\hat{G}_\Delta)$. 
{\bf (b)} follows from the margin assumption \ref{assump:margin} applied at threshold values $1/2- \gd$ and $1/2 + \gd$. 
\end{proof}


\newpage

\section{Algorithm for Setting 3: Bounded-rate with additional unlabelled samples}

\subsection{Details of Algorithm~3}
\label{appendix:setting3-details}

\paragraph{Outline of Algorithm~3.} 
Algorithm~3 also proceeds by constructing the set of active points $\X_t$, the set of classified points $\Xtc$,  the set of unclassified points $\Xtu$, and the set of discarded points $\Xtd$. In the beginning, it requests a set of unlabelled samples to estimate the marginal $P_X$. It then constructs the estimates of the true threshold values using the regression function estimates along with the empirical measure constructed from the unlabelled samples. It then updates the set of unclassified points based on the estimated threshold values. The point selection and cell refinement rules are unchanged from Algorithms~1 and~2. Algorithm~3 requests for more unlabelled samples when the error term in estimating the thresholds due to the unlabelled samples exceeds the error term due to the labelled samples. 


\paragraph{Steps of Algorithm~3.}
The algorithm proceeds in the following steps:
\begin{enumerate}[wide, labelwidth=!, labelindent=0pt]

\item For $t=0$, initialize $\X_0 = \{x_{0,0}\}$, $\X_0^{(u)}=\X_0$, $\X_0^{(c)} = \emptyset$,$\X_0^{(d)} = \emptyset$,  $u_0(x_{0,0}) = +\infty$, and $l_0(x_{0,0}) = -\infty$. Set  $h_{\max}(n) = \frac{\log n}{2\beta \log(1/\rho)}$.  Request $n$ unlabelled samples\footnote{We set the initial number of requested unlabelled samples equal to the total budget $n$ to ensure that $m \geq n$.} and construct the empirical measure $\hat{P}_X$ and slack term $s_t$ that represents the accuracy of $\hat{P}_X$ (see Prop.~\ref{prop:slack} in Appendix~\ref{appendix:setting3-proofs}). 

\item For $t\geq 1$, construct the upper and lower bounds, $u_t(\xhi)$ and $l_t(\xhi)$, for every $\xhi \in \X_t$, as in Algorithms~1 and~2. 

\item Define the piecewise constant function $f_t(\cdot)$ as in Algorithm~2. 

\item Sort the points/cells in the unclassified set $\X_t^{(u)}$ in ascending order of their $|f_t(\cdot)-1/2|$ values. Denote the ordered cells by $\Et{j}$ and their corresponding (ordered) center points by $\xtj{j}$. Introduce $k_{1,t} \coloneqq  \max\big\{k \geq 1 \; \mid \; \hat{P}_X\big(\cup_{j=1}^k \Et{j}\big) \leq \delta - s_t \big\}$ and use it to define the threshold $\hgone = f_t\big(\xtj{k_{1,t}}\big)$ and the set $S_1=\cup_{j=1}^{k_{1,t}}\Et{j}$. Similarly, introduce $k_{2,t} \coloneqq \text{min} \big\{k \geq 1 \; \mid \; \hat{P}_X\big(\cup_{j=1}^k \Et{j}\big) \geq \delta + s_t\big\}$ and use it to define the threshold $\hgtwo = f_t \big(\xtj{k_{2,t}}\big)$ and the set $S_2 = \cup_{j=1}^{k_{2,t}}\Et{j}$. 

\item Select a candidate point $\xhit$ as in Algorithms~1 and~2 and introduce $J_t = \It{1}\lp \xhit \rp = \max_{\xhi \in \Xtu} \It{1}\lp \xhi \rp$. 

\item If $e_t(n_{h_t,i_t}(t)) < V_{h_t}$ AND $h_t < h_{\max}(n)$, {\em refine} the cell, otherwise, call REQUEST\_LABEL.  
    
\item If $J_t \leq (s_t/C_2)^{1/\alpha_2}$, then {\em request unlabelled samples} and update $\hat{P}_X$ and  $s_t$ until $J_t >(s_t/C_2)^{1/\alpha_2}$. 

\item The set of unclassified points $\Xtu$ is updated according to Eq.~\eqref{eq:update_rule_alg2}.
%
%

\item Suppose the algorithm stops at time $t_n$. Then we construct the abstain region similar to Algorithm~2 but with $\hat{P}_X$ instead of the true marginal $P_X$. More specifically, with the definitions of $E'_{(j)}$ as in Step~9 of Algorithm~2, we define $k' \coloneqq \max\{k \mid \hat{P}_X\lp \cup_{j=1}^k E'_{(j)} \rp \leq \delta - s_{t_n}\}$. We then proceed to define $j'$ and $c'$ to ensure that the empirical measure of the abstaining region is exactly equal to $\delta - s_{t_n}$.

\item Finally, the classifier returned by the algorithm is 
\begin{equation}
\label{eq:classifier-algo3}
\hat{g}(x) = 
\begin{cases}
\Delta & \text{if } \; x \in \cup_{j=1}^{k'}E_{(j)}' \\
\big( (1-c')(1-j'), (1-c')j', c' \big) & \text{if } \; x \in E'_{(k'+1)} \\
1 & \text{if } \; u_{t_n}\lp \pi_{t_n}\lp x \rp \rp > 1/2  \; \text{ or } \; x \in \tilde{X}_n, \\
0 & \text{if } \; l_{t_n}\big(\pi_{t_n}(x)\big) < 1/2  \; \text{ and } \; x \not \in \tilde{\X}_n.
\end{cases}
\end{equation}
where $\pi_{t_n}(\cdot)$ is defined in~\eqref{eq:temp000}. 
\end{enumerate}


\begin{remark}
We have described the steps of Algorithm~3 assuming that~(DE) holds, and the bounds $C_2$ and $\alpha_2$ are known. In case (DE) does not hold, as we show in Appendix~\ref{appendix:setting3-proofs}, our analysis of   Algorithm~3 cannot guarantee faster rates of convergence for \emph{easy} problem instances (See Theorem~\ref{theorem:setting3}, Section~\ref{sec:discussion} and Appendix~\ref{appendix:discussion}). 
In this case, we can remove Step~6 of the algorithm and construct the estimate $\hat{P}_X$ based on $m = \mc{O}(n^2)$ samples, which can be drawn all at once at the beginning of the algorithm,  to ensure a uniform deviation bound on $\hat{P}_X$ of the order $1/n$. 
\end{remark}

\subsection{Pseudo-code of Algorithm~3}
\label{appendix:setting3-algo}
In this section, we report the pseudo-code of Algorithm~3 that was outlined and described in Section~\ref{subsec:setting3}. This is our active learning algorithm for the third abstention model, i.e., the bounded-rate setting with access to additional unlabelled samples from the marginal distribution $P_X$. 

\begin{algorithm}
\SetAlgoLined
\SetKwInput{Input}{Input}
\SetKwInput{Output}{Output}
\SetKw{Init}{Initialize}

\caption{Active learning algorithm for the bounded rate of abstention setting.}
\Input{$n$, $\delta$, $L$, $\beta$, $v_1$, $v_2$, $\rho$, $C_2$, $\alpha_2$} 
 
\Init{$t=1$, $n_e=0$, $\X_t = \{x_{0,1}\}$, $\Xtu = \X_t$, $\Xtc= \emptyset$, $h_{\max}(n) = \frac{\log(n)}{2\beta \log(1/\rho)}$}
 
\If{$n_e = 0$}
{
$u_t(x_{0,1}) = + \infty$, \quad $l_t(x_{0,1}) = - \infty$ 
}
\vspace{0.5em}
\While{$n_e \leq n$}{

\For{$x_{h,i}\in \Xtu$} 
{
update $u_t(\xhi)$, $l_t(\xhi)$, and $f_t(\xhi)$
}

define $k_{1,t}$, $k_{2,t}$, $\hgone$, $S_1$, $\hgtwo$, and $S_2$

\vspace{0.5em}

\tcc{Choose a candidate point with most uncertainty} 

$x_{h_t,i_t} \in \argmax_{x_{h,i}\in \Xtu} \It{1}(\xhi) = u_t(\xhi) - l_t(\xhi)$

\vspace{0.5em}
 
\tcc{Refine or Label}
 
\eIf{$\;e_t\big(n_{h,i}(t)\big)< L(v_1\rho^{h_t})^{\beta}$ {\bf and } $h_t< h_{\max}(n)$ }
{
$\Xtu \leftarrow \Xtu \setminus \{\xhit\}\cup \{ x_{h_t+1, 2i_t-1}, x_{h_t+1, 2i_t}\}$ \\
$u_t(x_{h_t+1,2i_t-1})\leftarrow u_t(\xhit);\qquad$  $l_t(x_{h_t+1,2i_t-1})\leftarrow l_t(\xhit)$ \\
$u_t(x_{h_t+1,2i_t})\leftarrow u_t(\xhit);\qquad\;\;\;\;$ $l_t(x_{h_t+1,2i_t})\leftarrow l_t(\xhit)$ 
}
{
call REQUEST\_LABEL\;

}
 
\vspace{0.5em}

\tcc{Request unlabelled samples}

\While{ $J_t \leq (s_t/C_2)^{1/\alpha_2}$}
{
request an unlabelled sample \\
update $\hat{P}_X$ and $s_t$ 
}

\vspace{0.5em}

\tcc{Update $\Xtu$ and $\Xtc$}

$\mathcal{Z}_t \leftarrow \big\{ \xhi \in \Xtu\, \mid \, f_t(\xhi) > \hgtwo + 3\It{1}(\xhit) \big\}$ \\
$\Xtu \leftarrow \Xtu \setminus \mathcal{Z}_t$, \quad $\Xtc \leftarrow \Xtc \cup \mathcal{Z}_t$ \\

$t \leftarrow t+1$
}
\Output{$\hat{g}$ defined by Eq.~\ref{eq:classifier-algo3}\\}
\label{alg:basic_alg3}
\end{algorithm}

\subsection{Proof of Theorem~\ref{theorem:setting3}}
\label{appendix:setting3-proofs}


\noindent
We now state the main result of this section, Theorem~\ref{theorem:setting3}, which provides an upper bound on the excess risk of the abstaining classifier constructed by Algorithm~3.

\begin{theorem}
\label{theorem:setting3}
Suppose the assumptions (MA) and (H\"O) hold with parameters $(C_0, \alpha_0)$ and $(L,\beta)$, respectively. 
Moreover, assume that the regression function $\eta(\cdot)$ is such that $|\eta(X)-1/2|$ has no atoms. Then, for large enough $n$, with probability at least $1-2/n$, the following statements are true for the classifier $\hat{g}$ defined by~\eqref{eq:classifier-algo3}:
\begin{enumerate}[wide, labelwidth=!, labelindent=0pt]
\item The classifier $\hat{g}$ is feasible for~\eqref{bayes_optimal_delta}, i.e.,~$P_X(\hat{G}_\Delta) \leq \delta$.

\item If $m \geq C'n^2$, for some $C'>0$, the following holds:
\begin{equation}
\label{eq:excess-risk-algo3-3}
    R(\hat{g}) - R(g_\delta^*) \leq 2C_0  J_{t_n}^{1+\alpha_0}
    , \quad \text{where } \; J_{t_n} = \tilde{\mc{O}}\big(n^{-\beta/(2\beta + D)}\big). 
\end{equation}
%
\item If the assumption (DE) also holds with parameters $(C_1, \alpha_1)$, we first define $C_2 = \min\{ C_1, C_0\}$, $\alpha_2 = \max \{\alpha_1, \alpha_0\}$, and $\tilde{D}$
denote the dimension term defined in Remark~\ref{remark:dimension}. 
\begin{itemize}
\item Then, for any $a>\tilde{D}$, we have 
%
\begin{equation}
\label{eq:excess-risk-algo3-2}
R(\hat{g}) - R(g^*_{\delta}) \leq 2C_0 J_{t_n}^{1+\alpha_0}  , \qquad\;\; \text{where} \quad \; J_{t_n} = \tilde{\mathcal{O}}\big( n^{-\beta/(2\beta + a)} \big).
\end{equation}
%
%
%
The hidden constant in~\eqref{eq:excess-risk-algo3-2}   depends on the parameters $L$, $v_1$, $v_2$, $\rho$, $\beta$, $C_0$ and $a$. 


\item Furthermore,  the additional number of unlabelled samples requested by the algorithm, denoted by $m_n$, is $\mc{O}(n^{2\alpha_2})$. 
is a function of the total budget $n$ and can be upper-bounded as
\begin{equation}
\label{eq:number-unlabelled}
m_n \leq\inf_{k >1}\Big( (4k)^{k/(k-1)} \frac{n^{(\alpha_2 k)/(k-1)}}{C_2^{2k/(k-1)}} \Big) \leq \frac{64n^{2\alpha_2}}{C_2^4 (Lv_1^{\beta})^{2\alpha_2}}. 
\end{equation}

\end{itemize}
\end{enumerate}
\end{theorem}

\begin{remark}
The first two statements of Theorem~\ref{theorem:setting3} imply that under the same assumptions  as those used in the Settings~1 and~2, we can achieve an excess risk that depends on the ambient dimension $D$ (see Eq.\eqref{eq:excess-risk-algo3-3}) with an additional $\mc{O}\lp n^2 \rp$ unlabelled samples. However, in order to exploit the \emph{easy} problem instances with small values of \emph{near-$\lambda$ dimension}, we shall require that the (DE) assumption also holds and the algorithm knows the parameters $C_2$ and $\alpha_2$. With these additional assumptions and information, we can achieve the same excess risk as in the infinite unlabelled samples framework of Algorithm~2, while only requiring a polynomial in the total budget $n$ number of unlabelled samples. The necessity of the (DE) assumption is discussed in Section~\ref{sec:discussion} and Appendix~\ref{appendix:discussion}. 
\end{remark}

\paragraph{Outline of the proof.} We first present Proposition~\ref{prop:slack} which is a uniform bound on the deviation of the empirical measure $\hat{P}_X$ from the true marginal $P_X$.
Next, we show in Lemma~\ref{lemma:algo3_lemma1} that the estimated thresholds $\hgone$ and $\hgtwo$ can be used to obtain lower and upper bounds on the true threshold value $\gd$. 
This result however, does not give us a measure of closeness of $\hgone$ and $\hgtwo$. As demonstrated through the counterexample in Appendix~\ref{appendix:discussion}, the difference between these two terms can potentially be large. As a result, without any additional assumption, we obtain convergence rates depending upon the ambient dimension $D$. 
Next, in Lemma~\ref{lemma:algo3_lemma2} we show that under additional detectability assumption, we can upper bound the difference between $\hgone$ and $\hgtwo$, which allows us to restrict the region of the input space searched by the algorithm. Using this we obtain the required convergence rates depending on a dimension term $\tilde{D}_{(3)}$ which is always smaller than $D$. Finally, in Lemma~\ref{lemma:algo3_lemma3}, we obtain an upper bound on the unlabelled sample requirement of our algorithm.

\begin{proposition}
\label{prop:slack}
Given $m$ unlabelled samples, we define the empirical measure of a set $E$ as $\hat{P}_X(E)\coloneqq\frac{1}{m}\sum_{j=1}^{m}\indi_{\{X_j\in E\}}$. Then, the event $\Omega_3 = \cap_{m\geq 1}\Omega_{3,m}$
, where $\Omega_{3,m}$ is defined below, occurs with probability at least $1-1/n$.
\begin{equation*}
\label{eq:conc1}
\Omega_{3,m} \coloneqq \Big\{\sup_{c>0}\big\{\big|\hat{P}_X\big(|f_t - 1/2| \leq c\big) - P_X\big(|f_t - 1/2| \leq c\big)\big|\big\} \leq s_m \Big\},
\end{equation*}
where the slack term $s_m$ is defined as
\begin{equation}
    s_m \coloneqq 2\sqrt{\frac{18\log(2\pi^2m^2n/3)}{m} }.
\end{equation}
\end{proposition}
\begin{proof}
For this inequality, we first note that the class of functions $\mathcal{F}_1 \coloneqq \{f_c:\mbb{R}\to \{0,1\}\; \mid \; f_c(x) =
\indi_{\{|x|\leq c \}},\ c \in \mbb{R}\}$,  has the VC dimension of $2$~\citep[\S~6.3.2]{shalev2014understanding}. This implies the following uniform convergence result with probability at least $1-6/(\pi^2m^2n)$, for $m$ samples $\{Z_j\}_{j=1}^m$ drawn i.i.d.~from any distribution $P_Z$~\citep[\S~28.1]{shalev2014understanding}:
\begin{align}
\sup_{f_c \in \mathcal{F}_1 } \lp \frac{1}{m}\left|\sum_{j=1}^m \Big( f_c(Z_j) - \mbb{E}_{P_Z}\big[ f_c(Z_j)\big] \Big)\right| \rp &\leq 2\sqrt{\frac{16 \log(em/2) + 2\log(2\pi^2m^2n/3)}{m} } \nonumber\\
&\leq 2\sqrt{\frac{18 \log(2\pi^2m^2n/3)}{m} } \coloneqq s_m. \label{eq:temp-AppB-3}
\end{align}
Now, we note that at any time $t$, conditioned on the set of labelled points, $f_t(\cdot)$ is a fixed function. Define the random variables $Z_j = f_t(X_j)-1/2$ for $j=1,2,\ldots,m$ and introduce the event
\[
\mathcal{E}_m  = \left \{ \sup_{f_c \in \mathcal{F}_1} \lp  \frac{1}{m} \left|\sum_{j=1}^m
\Big(f_c(Z_j) - E_{P_Z}\big[ f_c(Z_j)\big]\Big)\right|  \leq s_m \rp \right \}.
\]
Then, we have
\begin{equation*}
P\lp \mathcal{E}_m^c \rp = \mbb{E}\lb \indi_{\mathcal{E}_m^c} \rb 
= \mbb{E}\big[ \mbb{E}[ \indi_{\mathcal{E}_m^c} \big| S_l ] \big] = \mbb{E}\lb P( \mathcal{E}_m^c \big| S_l ) \rb \leq \mbb{E} \lb \frac{6}{m^2\pi^2n} \rb  = \frac{6}{m^2\pi^2n},
\end{equation*}
where the  inequality follows from~\eqref{eq:temp-AppB-3}. This proves that $P\lp \mathcal{E}_m \rp \geq P(\Omega_{3,m}) \geq 1 - \frac{6}{m^2\pi^2n}$.

\end{proof}

Using this above proposition, if the number of unlabelled samples available to the algorithm at any time $t$ is $m$, we can set the slack term at that time equal to $s_m$.

Our next two results obtain the bounds on the threshold values $\hgone$ and $\hgtwo$ estimated by the algorithm. 
\paragraph{Without detectability assumptions:} 
\begin{lemma}
\label{lemma:algo3_lemma1}
The follwoing bounds hold for $\hgone$ and $\hgtwo$:
\begin{equation}
\label{eq:algo3_eq1}
   \hgone \leq \gd \leq \hgtwo + J_t.  
\end{equation}
\end{lemma}

\begin{proof}
We first show the lower-bound in~\eqref{eq:algo3_eq1}. Since $\hat{P}_X\lp S_1 \rp \leq \delta - s_t$, using Proposition~\ref{prop:slack}, we obtain $P_X\lp S_1 \rp  \leq \delta$. 
Now, we may write
\begin{equation*}
S_1 \supseteq \{|f_t-1/2|<\hgone \} \supseteq \{|\eta -1/2| < \hgone \} = \{ |\eta - 1/2| \leq \hgone \}. 
\end{equation*}
Combining these results, we obtain
\begin{equation*}
P_X\lp |\eta-1/2| \leq \hgone \rp \leq P_X\lp S_1 \rp \leq \delta = P_X\lp 
|\eta -1/2|\leq \gd \rp, 
\end{equation*}
which implies $\hgone \leq \gd$. 

For proving the upper bound on $\gd$ in \eqref{eq:algo3_eq1}, we first note that by Proposition~\ref{prop:slack}, we have $P_X(S_2) \geq \delta$. Next, we have the following:
\begin{align}
S_2 &\subset \{|f_t-1/2| \leq \hgtwo\} \subset \{|\eta - 1/2| \leq \hgtwo + J_t \} \nonumber \\
\Rightarrow \delta \leq  P_X(S_2) & \leq P_X\lp |\eta - 1/2| \leq \hgtwo + J_t \rp. 
\label{eq:algo3_eq22}
\end{align}
The last inequality in~\eqref{eq:algo3_eq22} implies that $\hgtwo + J_t \geq \gd$. 
\end{proof}

Note that the above lemma ensures that the estimated thresholds $\hgone$ and $\hgtwo$ can be used to obtain an interval contining the true threshold $\gd$. 
However, this gives us no information about the length of the interval $[\hgone, \hgtwo + J_t]$ even for small values of $J_t$. 
Thus we use the trivial inclusion $\cup_{\xhi \in \Xtu} \X_{h,i} \subset \X$, and use the fact that the packing dimension of $\X$ is equal to $D$. We then proceed as in Lemma~\ref{lemma:algo2_lemma4} and Lemma~\ref{lemma:step1_2} to get an upper bound on $J_{t_n}$ of the form $\mc{O}\lp n^{\beta/(2D + \beta)}\rp $ with high probability. 
We can then use the (MA) condition as in Lemma~\ref{lemma:algo2_lemma5} to get the conclusion.

\paragraph{With detectability assumption:}

We next present a lemma, which tells us that under the additional $(DE)$ assumption, we can also show that the terms $\hgone$ and $\hgtwo$ are close to $\gd$.

\begin{lemma}
\label{lemma:algo3_lemma2} 
If the (DE) assumption holds, then we have the following:
\begin{align*}
    \hgtwo - 2J_t \leq \gd \leq \hgone + 3J_t.
\end{align*}
\end{lemma}

\begin{proof}

For the lower-bound, we first introduce the term $S_1' = S_1 \cup \Et{k_{1,t}+1}$. By definition of $k_{1,t}$, we know that $\hat{P}_X\lp S_1'\rp > \delta - s_t$, which further implies that $P_X\lp S_1'\rp > \delta - 2s_t$. Now, by the same reasoning as that used in the proof of Lemma~\ref{lemma:algo2_lemma3}, we know that $\left|f\lp \xtj{k_{1,t}+1}\rp\right| \leq \hgone + J_t$, which gives us the following sequence
\begin{align}
    S_1' &\subset \{|f_t -1/2| \leq \hgone + J_t \} subset \{|\eta-1/2| \leq \hgone + 2J_t \}\nonumber \\
    \Rightarrow \delta - 2s_t &\leq P_X\lp |\eta-1/2| \leq \hgone + 2J_t \rp \label{eq:algo3_eq2}. 
\end{align}

Now for any $z>0$, we have by the assumption~\ref{assump:detect} that $P_X\lp |\eta-1/2|\leq\gd - z\rp \leq \gd - 2C_1z^{\alpha_1}$, which for $z=\lp \frac{s_t}{C_1}\rp^{1/\alpha_1}$ gives us the following: 
\begin{equation*}
P_X\lp |\eta - 1/2|\leq \gd - z\rp \leq \delta - 2s_t \leq P_X\lp |\eta-1/2| \leq \hgone + 2J_t \rp. 
\end{equation*}
This implies the bound 
\begin{equation*}
\hgone \geq \gd - 2J_t - \lp \frac{s_t}{C_1}\rp^{1/\alpha_1} \geq 3J_t. 
\end{equation*}
where the last inequality follows from the rule used for requesting unlabelled samples.


Now, by the definition of $k_{2,t}$, we know that for $S_2' \coloneqq S_2\setminus \Et{k_{2,t}}$, we have $\hat{P}_X\lp S_2' \rp \leq \delta + s_t$ which by Proposition~\ref{prop:slack} implies that $P_X\lp S_2'\rp \leq \delta + 2s_t$. Proceeding as in the proof of Lemma~\ref{lemma:algo2_lemma3}, we can conclude that $f_t\lp\xtj{k_{2,t}-1}\rp \geq \hgtwo - J_t$. This itself implies that
\begin{align*}
\{ |f_t-1/2| < \hgtwo - J_t \} &\subset S_2' \\
\Rightarrow \{|\eta - 1/2| \leq \hgtwo - J_t \} &\subset S_2' \\
\Rightarrow P_X\lp \{|\eta - 1/2| \leq \hgtwo - J_t \}\rp & \leq \delta + 2s_t.
\end{align*}
By the detectability condition~(DE), for $z= \lp s_t/C_1\rp^{1/\alpha_1}$, we have $P_X\lp |\eta - 1/2| \leq \gd + z \rp \geq \delta + 2s_t$. This implies that $\hgtwo - J_t \leq \gd + z$, which combined with the rule used for requesting unlabelled samples implies that $\hgtwo \leq \gd + 2J_t$. 
\end{proof}




Having obtained these bounds on the threshold estimates $\hgone$ and $\hgtwo$, we can now proceed in a manner analogous to the proof of Theorem~\ref{theorem:setting2}. 

\begin{itemize}
\item We can show that for any $x \in \cup_{\xhi \in \Xtu}\Xhi$, we have $|\eta(x)-\lambda| \leq 6J_t \leq 12V_{h_t-1}$ for $\lambda \in \{1/2 - \gd, 1/2 + \gd \}$. 

\item This brings into play the dimension term $\tilde{D}$, defined in Remark~\ref{remark:dimension}, which rougly gives a measure of the packing dimension of the region explored by the algorithm near the threshold values. 

%
%
The dimension term can then be used as in Lemma~\ref{lemma:algo2_lemma4} to obtain a bound on $J_{t_n}$. 

\item Finally, we can combine the bound on the regression function estimation error along with the margin conditions to obtain the bounds on the excess risk similar to Lemma~\ref{lemma:algo2_lemma5}. 
\end{itemize}

It remains to obtain the bound on the number of unlabelled samples requested by the algorithm. 

\begin{lemma}
\label{lemma:algo3_lemma3}
The number of unlabelled samples requested by the algorithm, $m_n$, satisfies
\begin{equation}
\label{eq:num-unlab-data}
    m_n \leq \inf_{k >1}\lp \lp 4k\rp^{k/(k-1)} \frac{n^{(\alpha_2 k)/(k-1)}}{C_2^{2k/(k-1)}} \rp \leq 64 \frac{n^{2\alpha_2}}{C_2^4 (Lv_1^{\beta})^{2\alpha_2}}. 
\end{equation}
\end{lemma}
\begin{proof}
Since the  algorithm does not expand beyond level $h_{\max}(n)$, we must have at all times $t \leq t_n$ and $J_t > V_{h_{max}(n)}$. This implies that the we must have $s_t > C_2V_{h_{\max}(n)}$, for all $t$. The result then follows by the fact that $m \geq n$ and using the fact that for all $k>0$, we have $\log (m) \leq k \lp m^{1/k} - 1 \rp$. The final inequality in~\eqref{eq:num-unlab-data} is obtained by setting $k=2$. 
\end{proof}


\newpage
\section{Details of the Adaptive Scheme (Section~\ref{sec:adaptivity})}
\label{appendix:adaptivity}

In this section, we elaborate on the adaptive scheme introduced in Section~\ref{sec:adaptivity} of the main text. More specifically, to simplify the presentation, we will restrict our attention to the fixed-cost setting with membership query model. Having obtained the adaptive scheme for this combination, we can appropriately modify it for other abstention schemes and active learning models. 

As mentioned in Section~\ref{sec:adaptivity},  the first modification required by the adaptive scheme is in the point selection rule. Here we select one point from every level $h$ in the set of active points. Since we have $h_{\max} = \log n$, this modification results in an additional polylogarithmic factor in the estimation error of the regression function, and hence the excess risk bound.
    
The second and more important modification is in the rule for refining a cell $\X_{h,i}$. In the case of known smoothness, we refine a cell if the stochastic uncertainty term, i.e., $e_t(n_{h,i})$, is roughly of the same order as the variation term $V_h = L(v_1\rho^h)^{\beta}$. This implies two things:
\begin{itemize}[leftmargin=!]
\item If the cell $\X_{h,i}$ is refined by the algorithm, it means that $\min \{ \sup_{x \in \X_{h,i}}|\eta(x) - \lambda|, \; \sup_{x \in \X_{h,i}}|\eta(x)-1+\lambda|\} = \mc{O}\lp V_h\rp$, and 

\item the number of times the cell $\X_{h,i}$ was queried by the algorithm before refining, denoted by $n_{h,i}$, satisfies $n_{h,i} = \mc{O}\lp \frac{\log(n)}{V_h^2}\rp$. 
\end{itemize}

Thus to obtain the same convergence rates on excess risk, it suffices to design a scheme which satisfies the above two properties for a given cell. We begin be first recalling a definition of \emph{quality} from \cite{slivkins2011multi}, suitably modified for our problem 
\begin{definition}
\label{def:quality}
Given $\X=[0,1]^D $, a regression function $\eta(\cdot):\X \mapsto [0,1]$ and a tree of partitions $(\X_h)_{h \geq0}$. For any cell $\X_{h,i}$, define $V_{h,i} \coloneqq \sup_{x_1,x_2 \in \X_{h,i}} \eta(x_1)-\eta(x_2)$, and define $\tilde{\eta}_{h,i} \coloneqq \int_{\X_{h,i}}\eta d\nu$ where $\nu$ is the Lebesgue measure on $\X$.  
We say the pair $\lp \eta, (\X_h)_{h\geq 0}\rp$ have \emph{quality} $q \in (0,1)$ if the following holds: for any cell $\X_{h,i}$, there exist two cells $\X_{h',i_1}$ and $\X_{h',i_2}$ subsets of $\X_{h,i}$ such that \textbf{1)} $\nu\lp \X_{h',i_j}\rp \geq q \nu\lp \X_{h,i}\rp$ for $j=1,2$ and \textbf{2)} $\tilde{\eta}\lp \X_{h',i_1} \rp - \tilde{\eta}\lp \X_{h',i_2}\rp \geq \frac{V_{h,i}}{2}$.
\end{definition}

We now state the additional assumption required by our adaptive scheme:
\paragraph{ (QU):} We assume that the pair $(\eta, (\X_h)_{h\geq 0})$ have quality $q>1/\log n$ where $n$ is the label budget. 

Next we present our adaptive scheme used for refining a cell. 
\paragraph{Adaptive Scheme for refining one cell.}
To simplify notation, we will refer the cell under consideration as $E$ (instead of $\X_{h,i}$), and use $W = L(v_1\rho^h)^{\beta}$ (instead of $V_h$) for the rest of the section. Introduce the partitions of $E$, denoted by $\mc{E}_1$, $\mc{E}_2$, \ldots, $\mc{E}_k$, where $k = \lceil \frac{\log \lp v_1^D\log n\rp}{D \log (1/\rho)}\rceil$ (the sets $\mc{E}_j$ consist of points in $\X_{h'}\cap \X_{h,i} $ for $h<h'\leq h+k$). Note that  any $A$ in $\mc{E}_j$ has the property that $(v_2\rho^j)^D \leq \text{Vol}(A)/\text{Vol}(E) \leq (v_1\rho^j)^D$, and by the definition of $k$, this ratio is smaller than $1/\log n$ for $A \in \mc{E}_k$. We will use $W_j$ to represent the corresponding upper bound on the variation of the sets in partitions $\mc{E}_j$, for $j=1,2,\ldots, k$. 

Since we are working in the membership query model, we assume that we can request the labels at $N_1 = |\mc{E}_k| = \mc{O}\lp \log n\rp$ points at a time, with exactly one point drawn uniformly from each set of $\mc{E}_k$. This is just to simplify the presentation of the scheme. In a pool-based or stream-based model, we can get an equivalent result by using martingale arguments. 

The adaptive scheme proceeds as follows:
\begin{itemize}[leftmargin=!]
    \item At $t=1$, request $N_1$ labelled samples from the cell $E$. Set $n_1(t)$ equal to $N_1$. 
    
    \item  \textbf{Estimate the variation in the cell.} 
    For all $t \geq 2$, we define the term $e_j(t) = \frac{c_1}{\sqrt{n_1(t) 2^{-j}}}$ for all $1 \leq j \leq k$ (While running the algorithm we must have $c_1 = \mc{O}\lp \log n \rp$). Next, we introduce the following terms
    \begin{align*}
\heub_j \coloneqq & \max_{A \in \mc{E}_j} \hat{\eta}_t(A) \quad \text{and} \quad \helb_j \coloneqq \min_{A \in \mc{E}_j} \hat{\eta}_t(A) \quad   
\text{where } \quad \hat{\eta}_t(A) \coloneqq \frac{1}{n_1(t)}\sum_{i=1}^{n_1(t)}Y_i \indi_{\{X_i \in A\}}, \\ 
\tilde{\eta}(A) & \coloneqq \int_{A}\eta(x)d\nu(x) \quad \text{for all} A \in \mc{E}_j \text{ for all } 1\leq j \leq k\\
w_j \coloneqq & \max_{A_1,A_2 \in \mc{E}_j} \tilde{\eta}(A_1) - \tilde{\eta}(A_2), \quad \text{and} \quad \hat{w}_j \coloneqq \; \heub_j - \helb_j \\
\end{align*}
The term $e_j(t)$ is such that we have with probability at least $1-1/n^2$, for all $t \geq 1$, $\hat{\eta}_t(A) - $    

With these definitions at hand, we construct an appropriate estimate of the variation of the     
    
    We then define the term $\hat{j}(t)$ as follows:
    \begin{align*}
        \hat{j}_t \coloneqq \min \{ 1 \leq j \leq k \; : \; |\hat{w}_j - \hat{w}_i|\leq 4 e_t(i) \quad  \forall j \leq i \leq k\}. 
    \end{align*}
    Note that $\hat{j}_t$ is well defined, since the condition required in its definition is always satisfied at $k$. We next define $j^*_t$ as the smallest values of $j$ at which $e_j(t)$ is larger than $W_j$. Then, we can check that, $|\hat{w}_{\hat{j}_t} - w_{j^*_t}|\leq 6e_t(j^*_t)$. Furthermore, since $W \leq w_{j^*_t} + 2W_{j^*_t} \leq w_{j^*_t} + 2e_t(j^*_t)$, it implies that $W \leq \hat{w}_{\hat{j}_t} + 8 e_t(j^*_t)$. Using this, we can construct an upper bound on the variation of the regression function in the cell $E$ as $\hat{W}_t \coloneqq \hat{w}_{\hat{j}_t} + 6e_t(k)$, since $j^*_t \leq k$. 
    
    Since $\hat{W}_t \geq W$, this upper bound can be used  to update the sets $\Xtu$ and $\Xtc$, similar to the way in which $V_h$ was used by Algorithm~1. 
    
   \item \textbf{Stopping Rule.} Next we define the stopping rule as follows:  \emph{Refine} the cell if $\hat{w}_{\hat{j}_t} - 8 e_t(k) \geq 0$, else request another $N_1$ samples. 
    
\end{itemize}

We next state the lemma, which tells us that the adaptive scheme ensures that the two conditions mentioned at the beginning of this section are satisfied.

\begin{lemma}
\label{lemma:adaptivity}
If the above scheme refines the cell at time $t$, then we have the following:
\begin{align*}
    n_1(t) \geq \frac{4c_1^2}{W^2} \quad \text{and} \quad n_1(t) \leq N_1 + \frac{256c_1^2 \log n}{W^2}. 
\end{align*}
\end{lemma}

\begin{proof}
If the cell is refined by the above scheme at time $t \geq 2$, then the following sequence is true at $t$:
\begin{align*}
|w_{j^*_t} - \hat{w}_{\hat{j}_t}| &\leq 6e_t(j^*_t) \Longrightarrow w_{j^*_t} \geq \hat{w}_{\hat{j}_t} - 6e_t(j^*_t) \\
\Longrightarrow w_{j^*_t} - 2e_t(j^*_t) & \geq \hat{w}_{\hat{j}_t} - 8 e_t(j^*_t) 
\Longrightarrow w_{j^*_t} - 2 e_t(j^*_t)  \geq \hat{w}_{\hat{j}_t} - 8e_t(k) \stackrel{(i)}{\geq } 0, 
\end{align*}
where \tbf{(i)} in the above display follows from the stopping rule. This implies that we have the following 
\begin{align*}
    W & \geq w_{j_t^*} \geq 2e_t(j^*_t) \geq 2e_t(1) \\
   \Rightarrow W & \geq 2 c_1 \frac{1}{\sqrt{n_1(t)}} 
   \Rightarrow n_1(t)  \geq \frac{4c_1^2}{W^2}, 
\end{align*}
which proves the first part of the lemma. 

Next, since the cell was not refined at time $t-1$, it means that
\begin{align*}
    \hat{w}_{\hat{j}_{t-1}} & < 8e_{t-1}(k) \; \Longrightarrow \hat{w}_{\hat{j}_{t-1}} + 6e_{t-1}(j^*_{t-1}) \leq 14 e_{t-1}(k) \\
    \Longrightarrow W & \leq w_{j^*_{t-1}} + 2e_{t-1}(j^*_{t-1}) \leq \hat{w}_{\hat{j}_{t-1}} + 8 e_{t-1}(j^*_{t-1}) < 16e_{t-1}(k) \\
    \Longrightarrow n_1(t-1) & \leq \frac{256c_1^2 \log n}{W^2} \; \Longrightarrow n_1(t)  \leq N_1 + \frac{256c_1^2 \log n}{W^2}. 
\end{align*}

\end{proof}

\paragraph{Steps of the adaptive version of Algorithm~1.}
We now state the main steps of the adaptive version of Algorithm~1 in the membership query model:
\begin{itemize}
    \item At any time $t$, we maintain the sets $\X_t$, $\Xtu$ and $\Xtd$. 
    \item In each round, for all $h \leq h_{\max}$, we select a candidate point from $\X_t \cap \X_h$ with the largest value of $\hat{\eta}$, i.e., the empirical mean value of $\eta$ in the cell. 
    
    \item For every candidate point, if the stopping rule is not satisfied, we request the label of $N_1$ points from the cell. Thus in each round, the algorithm may request up to $h_{\max}N_1 = \mc{O}\lp (\log n)^2\rp$ labels. 
    
    \item If the stopping condition for a cell $\X_{h,i}$ is satisfied, we compute the following upper and lower bounds: $ u_t(x_{h,i}) \coloneqq \hat{\eta}(x_{h,i}) + e_t(x_{h,i}) + \hat{w}_{\hat{j}_t} + 6e_t(x_{h,i})\sqrt{\log n}$ and $l_t(x_{h,i})\coloneqq \hat{\eta}(x_{h,i}) - e_t(x_{h,i}) - \hat{w}_{\hat{j}_t} - 6e_t(x_{h,i})\sqrt{\log n}$. Using these upper and lower bounds on the function value of the cell, we update the sets $\Xtu$ and $\Xtc$ as before. 
\end{itemize}

\begin{remark}
\label{remark:adaptive1}
Lemma~\ref{lemma:adaptivity} along with the above steps imply two things: \tbf{1)} The stopping rule ensures that no cell at level $h$ of the tree will be evaluated more than $\mc{O}\lp \frac{\log n}{V_h^2}\rp$ times, and \tbf{2)} for any unclassified cell at level $h$, the $I_t^{(1)}$ value will be no larger than $2V_{h-1}\lp 1 + \sqrt{\log n}\rp$. Plugging these bounds in the analysis of Algorithm~1, we can recover the same upper bounds on the excess risk. 
\end{remark}

\newpage 
\section{Proof of Lower Bound}
\label{appendix:lower_bound}
\subsection{Proof of Lemma~\ref{lemma:lower_bound1}}

\label{appendix:lower_bound_lemma}

[In this section, we use the notation $\int_A f d\mu$ as a shorthand for $\int_A f(x)d\mu(x)$ for the integral of function $f$ with respect to some measure $\mu$ over some set $A$.]\\

We first observe the following:
\begin{align*}
R_\lambda(g) - R_\lambda(\gls) & = \int_{G_\lambda}\lambda dP_X + \int_{G_0}\eta dP_X 
+ \int_{G_1}(1-\eta)dP_X \\
& \quad -  \int_{\gls}\lambda dP_X - \int_{G_0^*}\eta dP_X - \int_{G_1^*}(1-\eta)dP_X \\
& = \int_{G_\lambda \cap G_0^*}\lp \lambda - \eta \rp dP_X + \int_{G_\lambda \cap G_1^*} \lp \lambda - 1 + 
\eta \rp dP_X + \int_{\gls \cap G_0}\lp \eta + \lambda \rp dP_X\\
& \quad + \int_{\gls \cap
G_1}\lp 1 -\eta - \lambda \rp dP_X + \int_{G_0\cap G_1^*} \lp 2\eta - 1\rp dP_X
+ \int_{G_0^*\cap G_1} \lp 1 - 2\eta \rp dP_X \\
& \coloneqq T_1 + T_2 + T_3 + T_4 + T_5 + T_6. 
\end{align*}

We now consider the six terms separately. 
\begin{itemize}
    \item By definition of $G_1^*$, we know that $\eta \geq 1 - \lambda$ in this 
        set. This implies that the integrand in $T_5$ is at least $1-2\lambda \geq 0$. 
        Thus we can lower bound $T_5$ with $0$. The term $T_6$ can similarly be 
        shown to be non-negative. 

    \item To lower bound the term $T_1$, we partition $G_{0}^*$ into two regions:
        $G_{0,a}^*$ which is close to the boundary, and $G_{0,b}^*$ which is the 
        region away from the boundary. 
        \begin{align*}
            G_{0,a}^* & \coloneqq \{ x \in G_0^* \; \mid \; \eta(x) \geq \lambda - t \}, \text{ and } \qquad G_{0,b}^*  \coloneqq G_0^* \setminus G_{0,a}^*, 
        \end{align*}
        where $t>0$ will be decided later. 
        In the set $G_\lambda \cap G_{0,b}^*$, we have $\lambda - \eta \geq t$, 
        which implies that 
        \begin{align*}
            T_1 & = \int_{G_\lambda \cap G_{0}^*}\lp \lambda - \eta \rp dP_X 
                 \geq \int_{G_\lambda \cap G_{0,b}^*} \lp \lambda - \eta \rp dP_X 
                 \geq t P_X \lp G_\lambda \cap G_{0,b}^* \rp \\
                & \geq t \lp P_X \lp G_\lambda \cap G_0^* \rp - P_X\lp G_{0,a}^* \rp \rp 
                 \stackrel{(i)}{\geq} t P_X\lp G_\lambda \cap G_0^* \rp - C_0 t^{1 + \alpha_0},  
        \end{align*}
        where the inequality $(i)$ follows from the margin condition. 

    \item To lower bound the term $T_2$, we introduce the sets
        $G_1^*$ into $G_{1,a}^* \cup G_{1,b}^*$ where $G_{1,a}^* \coloneqq \{ x \in G_1^* \mid \eta(x) \leq 1-\lambda + t\}$ and $G_{1,b}^* \coloneqq G_1^* \setminus G_{1,a}^*$. 
        We then have:
        \begin{align*}
             T_2 & = \int_{G_\lambda \cap G_{1}^*}\lp \lambda - 1 + \eta \rp dP_X 
                 \geq \int_{G_\lambda \cap G_{1,b}^*} \lp \lambda - 1 + \eta \rp dP_X 
                 \geq t P_X \lp G_\lambda \cap G_{1,b}^* \rp \\
                & \geq t \lp P_X \lp G_\lambda \cap G_1^* \rp - P_X\lp G_{1,a}^* \rp \rp 
                 \geq t P_X\lp G_\lambda \cap G_1^* \rp - C_0 t^{1 + \alpha_0}.         
        \end{align*}
    
    \item To lower bound $T_3$ we introduce $G_{\lambda, a}^* \coloneqq \{x \in \gls \mid \eta(x) \leq \lambda + t\}$, and $G_{\lambda, b}^* \coloneqq \gls \setminus G_{\lambda, a}^*$. Then we have the following:
    \begin{align*}
        T_3 & \coloneqq \int_{G_0 \cap \gls}(\eta - \lambda)dP_X \geq \int_{G_0 \cap G_{\lambda,b}^*}\lp \eta - \lambda \rp dP_X  \geq t P_X\lp G_0 \cap G_{\lambda, b}^* \rp \\
 &\geq t\lp P_X\lp G_0 \cap \gls \rp - P_X\lp G_{\lambda, a}^* \rp \rp  \geq tP_X\lp G_0 \cap \gls \rp - C_0 t^{\alpha_0 + 1}. 
    \end{align*}
    
    \item Finally, to lower bound the term $T_4$, we introuce $G_{\lambda,c}^* \coloneqq \{x \in \gls \mid \eta(x) \geq 1-\lambda - t\}$, and $G_{\lambda, d}^* = \gls \setminus G_{\lambda,c}^{*}$. Then we have
    \begin{align*}
        T_4 & \coloneqq \int_{G_1 \cap \gls}(1-\eta - \lambda)dP_X \geq \int_{G_1 \cap G_{\lambda, d}^*}(1-\eta - \lambda)dP_X \geq t P_X\lp G_1 \cap G_{\lambda, d}^* \rp  \\
         &\geq t \lp P_X\lp G_1 \cap \gls \rp - P_X\lp G_{\lambda, c}^* \rp \rp  \geq t P_X\lp G_1 \cap \gls \rp - C_0t^{\alpha_0 + 1}. 
    \end{align*}
    
\end{itemize}
Combining the above we have the following:
\begin{align}
    R_\lambda(g) - R_\lambda(g_\lambda^*) & \geq t \lp P_X\lp G_\lambda \cap (\gls)^c \rp  + P_X\lp \gls \cap G_\lambda^c \rp \rp- 4C_0 t^{1+\alpha_0} \nonumber\\
    & = tP_X\lp G_\lambda \triangle \gls \rp - 4C_0 t^{1 + \alpha_0}. \label{eq:proof_lb1}
\end{align}

The result then follows by setting $t$ such that $tP_X\lp G_\lambda \triangle \gls \rp = 5C_0t^{1+ \alpha_0}$, which leads to the following:
\begin{align*}
    R_\lambda(g) - R_\lambda(g_\lambda^*) & \geq C_0 \lp \frac{P_X\lp G_\lambda \triangle \gls\rp }{5 C_0}\rp^{(1+\alpha_0)/\alpha_0} \\
    &= \lp \frac{1}{5}\rp^{(1+\alpha_0)/\alpha_0}\lp \frac{1}{C_0}\rp^{1/\alpha_0} P_X\lp G_\lambda \triangle \gls \rp^{(1+\alpha_0)/\alpha_0}\\
    & \coloneqq c P_X\lp G_\lambda \triangle \gls \rp^{(1+\alpha_0)/\alpha_0}
\end{align*}

\subsection{Proof of Theorem~\ref{theorem:lower_bound2}}
\label{appendix:lower_bound2}

We follow the  general scheme for obtaining lower bounds in nonparametric learning problems used in  prior work such as \citep{audibert2007fast, minsker2012plug}. 
This method  involves constructing a set of \emph{hard} problem instances which are (1) sufficiently well separated in terms of some \emph{pseudo-metric}, and (2) sufficiently close together in terms of some statistical distance (such as KL divergence or $\chi^2$ distance).
Once we have such a construction, we can employ Theorem~2.5 of \citep{tsybakov2008introduction} (recalled below as Theorem~\ref{theorem:tsybakov1}) to get a lower bound on the distance in terms of the pseudo-metric for any estimator.
Finally, we can use the comparison lemma (Lemma~\ref{lemma:lower_bound1}) to convert this to a lower bound on the excess risk.

\begin{theorem}[Theorem~2.5 of \citep{tsybakov2008introduction}]
\label{theorem:tsybakov1}
Assume that for  $\tilde{M} \geq 2$, $\Theta =  \{\theta_1, \ldots, \theta_{\tilde{M}}\}$,   $\tilde{d}$ is a pseudo-metric on $\Theta$, and $\{P_{\theta_j} \mid \theta_j \in \Theta\}$ is a collection of probability measures such that:
\begin{itemize}
    \item $\tilde{d}\lp \theta_i, \theta_j\rp \geq 2s >0$ for all $1 \leq i,j \leq \tilde{M}$. 
    \item $P_{\theta_i}<< P_{\theta_0}$ for all $1 \leq i \leq \tilde{M}$. 
    \item $\frac{1}{\tilde{M}}\sum_{j=1}^{\tilde{M}}D_{KL}\lp P_{\theta_j}, P_{\theta_0} \rp \leq a \log\lp \tilde{M} \rp$ for $0<a<1/8$.
\end{itemize}
Then we have for $\tilde{M} \geq 10$,  
\begin{align*}
    \inf_{\hat{\theta}} \sup_{\theta \in \Theta} P_{\theta}\lp \tilde{d}\lp \hat{\theta}, \theta \rp \geq s \rp \geq \frac{1}{4}
\end{align*}
where the infimum is over all estimators $\hat{\theta}$ constructed using samples from $P_\theta$. 
\end{theorem}

We now describe the construction of the regression functions. First, given $\mc{X} = [0,1]^D$, for some $\epsilon>0$ to be decided later, we partition $\mc{X}$ into hypercubes of side $\epsilon$, and denote by $M = (1/\epsilon)^D$ the number of such hypercubes. Let $V$ be the set of centers of the hypercubes, i.e, $V= \{z_1, z_2, \ldots, z_M\}$, and let $\pi: \mc{X} \mapsto V$ denote the projection operator onto $V$. 
\paragraph{Choose appropriate subsets of the  input space.} Assuming $D \geq 2$, let $e_1, e_2, e_3$ and $e_4$ denote any four corner points of $\mc{X} = [0,1]^D$.
We define the following subsets of the space $\mc{X}$
\begin{align*}
Q_j & \coloneqq \{x \in \mc{X} \mid \|x - e_j\| \leq 1/3 \} \quad \text{for } j=1,2,3 \text{ and }4.
\end{align*}
For $\epsilon$ small enough, we note that there exists a constant $c_1>0$ such that the number of hypercubes contained inside each $Q_j$, denoted by $M_j$,  can be lower bounded by $c_1M$. (Note that by symmetry $M_1=M_2=M_3=M_4$, so we will use $\tilde{M}$ to denote any of $M_j$). 
We will use $V_j = \{ z_{j,1}, z_{j,2}, \ldots, z_{j,\tilde{M}}\}$ to denote the centers of the hypercubes contained in $Q_j$, and $Y_j \coloneqq \bigcup_{z \in V_j}B_{\infty}(z, \epsilon/2)$ to denote the union of all the hypercubes strictly contained in $Q_j$. Here $B_{\infty}(z, \epsilon/2)$ denotes the hypercube with center $z$ and side $\epsilon$.

\paragraph{Define  the regression function.} Let $u:[0,\infty) \mapsto [0,1]$ be a function defined as $u(z) = \min \{(1-z)^{\beta}, 0 \}$.
Note that $u$ satsifies the following properties: (1) $u(0) = 1-u(1) = 1$, (2), $u(z) =  0$ for $z\geq 1$, and (3) $u$ is $(1,\beta)$ H\"older continuous for $0<\beta\leq 1$.

For any $z \in S$, we define the function $\varphi_z(x) = L \lp \epsilon/2\rp ^\beta u \lp (2/\epsilon)\|x - z \|\rp$.
By construction, the function $\varphi_z$ is  is $(L, \beta)$ H\"older continuous. Furthermore, we assume that $\epsilon$ is small enough to ensure that $L(\epsilon/2)^{\beta} < 1/2 - \lambda$. 

For any $\Vec{\sigma}^{(j)} \in \{-1,1\}^{\tilde{M}}$, for $j=1,2$ we introduce the notation $\Vec{\sigma} = \lp \Vec{\sigma}^{(1)}, \Vec{\sigma}^{(2)}\rp \in \{-1,1\}^{2\tilde{M}}$.
Next we define  $\eta_{\Vec{\sigma}}(x) = \lambda + \sum_{i=1}^{\tilde{M}} \sigma^{(1)}_i \varphi_{z_{1,i}}(x)$ for $x \in Y_1$ and $1-\lambda + \sum_{i=1}^{\tilde{M}}\sigma^{(2)}_i \varphi_{z_{2,i}}(x)$ for $x \in Y_2$.
For $x$ lying in $Q_1\setminus Y_1$ and $Q_2 \setminus Y_2$, we assign $\eta_{\Vec{\sigma}}(x)$ the values $\lambda$ and $1-\lambda$ respectively.

Furthermore, we assign $\eta_{\Vec{\sigma}}(x) = 1$ for $x \in Q_3$ and $\eta_{\Vec{\sigma}}(x)=0$ for $x \in Q_4$.

It remains to specify the values of $\eta_{\Vec{\sigma}}(\cdot)$ in the region $\mc{X} \setminus \lp \bigcup_{j=1}^4 Q_j \rp$.
For any $A \subset \mc{X}$ and $x \in \mc{X}$, we use $d_A(x) \coloneqq \inf \{\|y-x\| \; \mid \; y \in A\}$ to represent the distance of the point $x$ from the set $A$. We also introduce the terms $z_1 = \lp \frac{1/2 - \lambda}{L}\rp^{1/\beta}$ and $z_2 = \lp \frac{1}{2L}\rp^{1/\beta}$, and assume that $L \geq 3$ which ensures that $z_1 \leq z_2 \leq 1/6$.
Now for all $x \in \mc{X}\setminus \bigcup_{j=1}^4Q_j$, we define 
\begin{align*}
    \eta_{\Vec{\sigma}}(x) = \begin{cases}
    \lambda + L u\lp 1 - d_{Q_1}(x) \rp &\text{if } x: d_{Q_1}(x)\leq z_1 \\
 1- \lambda - L u \lp 1 - d_{Q_2}(x) \rp &\text{if } x: d_{Q_2}(x) \leq z_1 \\
 1 - L u \lp 1 - d_{Q_3}(x) \rp &\text{if } x: d_{Q_3}(x) \leq z_2 \\
 L u \lp 1 - d_{Q_4}(x) \rp &\text{if } x: d_{Q_4}(x) \leq z_2 \\
 1/2 &\text{otherwise}
    \end{cases}
\end{align*}
This completes the definition of the regression function at all points in $\mc{X}$. By construction, we have that for any $\Vec{\sigma} \in \{-1,1\}^{2\tilde{M}}$, the regression function $\eta_{\Vec{\sigma}}$ is $(L,\beta)$ H\"older continuous for $0<\beta\leq 1$ and $L\geq 3$.


\paragraph{Define the marginal $P_X$.} Next, we need to define a marginal such that the margin condition is satisfied with exponent $\alpha_0>0$.
For this we can proceed as in \citep[\S~6.2]{audibert2007fast} and for some $w<(1/(2\tilde{M}) )$, define the density of the marginal w.r.t. the Lebesgue measure as follows:
\begin{align*}
    p_X(x) = \begin{cases}
   \frac{w \indicator_{B(\pi(x),\epsilon/4)}(x)}{\text{Vol}\lp B(\pi(x),\epsilon/4)\rp} & \text{ for } x \in Y_1 \cup Y_2 \\
  \frac{1 - 2\tilde{M} w}{2\text{Vol}\lp Q_j \rp} & \text{ for } x \in Q_j, \text{ for } j = 3,4 \\
    0 & \text{ otherwise}.
   \end{cases} 
\end{align*}
We can now check that the joint distribution thus defined satisfied the Margin condition for a given exponent $\alpha_0>0$  with constant $C_0 = \lp 8/3\rp^{\beta \alpha_0}$, if we have $\tilde{M}w = \mc{O}\lp \epsilon^{\alpha_0 \beta}\rp$.

\paragraph{Apply Theorem~\ref{theorem:tsybakov1}.}
In order to apply Theorem~\ref{theorem:tsybakov1}, we proceed as follows:
\begin{itemize}
    \item Let $\Sigma$ denote the set $\{-1,1\}^{2\tilde{M}}$. Then by \emph{Gilbert-Varshamov bound} \citep[Lemma~2.9]{tsybakov2008introduction}, we know that there exists a subset of $\Sigma$, denoted by $\tilde{\Sigma}$, such that $|\tilde{\Sigma}|\geq 2^{\tilde{M}/4}$, $\Vec{\sigma}_0 = (1,1,\ldots,1) \in \tilde{\Sigma}$, and for any $\Vec{\sigma}_1, \Vec{\sigma}_2 \in \tilde{\Sigma}$, we have $d_H(\Vec{\sigma}_1, \Vec{\sigma}_2) \geq \tilde{M}/4$. Here $d_H(\cdot, \cdot)$ denotes the Hamming distance. 
    
    \item Let $\mc{P}'$ denote the class of joint distributions $P_{\Vec{\sigma}}$ with marginal $P_X$, and conditional distribution $\eta_{\Vec{\sigma}}$ for $\Vec{\sigma} \in \tilde{\Sigma}$.
    For any two $P_{\Vec{\sigma}_1}$ and $P_{\Vec{\sigma}_2}$ in $\mc{P}'$, we introduce the pseudo-metric $\tilde{d}$ defined as $\tilde{d}\lp P_{\Vec{\sigma}_1},P_{\Vec{\sigma}_2} \rp \coloneqq P_X\lp \text{sign}\lp  \eta_{\Vec{\sigma}_1} - \lambda \rp \neq \text{sign}\lp \eta_{\Vec{\sigma}_2} - \lambda \rp \rp + P_X\lp \text{sign}\lp \eta_{\Vec{\sigma}_1} -1 + \lambda \rp \neq \text{sign}\lp \eta_{\Vec{\sigma}_2} - 1 + \lambda \rp \rp$. 
    
    Thus, by the properties of $\tilde{\Sigma}$, we get that for any $\Vec{\sigma}_1, \Vec{\sigma}_2 \in \tilde{\Sigma}$, we have 
    \begin{equation*}
        \tilde{d}\lp P_{\Vec{\sigma}_1}, P_{\Vec{\sigma}_2}\rp \geq \frac{\tilde{M}w}{4}.
    \end{equation*}
    
    \item Next, by using Eq.(10) of  \citep{minsker2012plug}, we can upper bound the average KL divergence between the distributions in $\mc{P}'$ after $n$ label requests by any active learning algorithm:
    \begin{align*}
    D_{KL}\lp P_{\Vec{\sigma}_1}, P_{\Vec{\sigma}_2} \rp \leq 32n L^2 \lp \frac{\epsilon}{2}\rp^{2\beta}. 
    \end{align*}
  If we select, $\epsilon = c_2 n^{-1/(D + 2\beta)}$, with $c_2$ small enough (a suitable value is $c_2 = \lp (4^{\beta}c_1)/(32^2 L^2)\rp^{1/(D + 2\beta)}$),  we have \begin{align*}
      D_{KL}\lp P_{\Vec{\sigma}_1}, P_{\Vec{\sigma}_2}\rp & \leq \frac{\tilde{M}}{4} \leq \frac{1}{8}\log\lp |\tilde{\Sigma}|\rp, 
  \end{align*} 
   as required by Theorem~\ref{theorem:tsybakov1}.  
    
\end{itemize}

Since all the conditions of Theorem~\ref{theorem:tsybakov1} are satisfied by our construction, we can conclude that for any active learning algorithm $\hat{\eta}$, we have 
\begin{align*}
    \inf_{\hat{\eta}} \sup_{ (\eta, P_X) \in \mc{P}' } \pr \lp P_X\lp \text{sign}\lp \hat{\eta}-\kappa \rp \neq \text{sign}\lp \eta - \kappa \rp \text{ for } \kappa \in \{\lambda, 1-\lambda\} \rp \geq c_3 n^{-(\alpha_0 \beta)/(D + 2\beta)} \rp \geq \frac{1}{4}.
\end{align*}

\paragraph{Apply the comparison inequality (Lemma~\ref{lemma:lower_bound1}).}
Finally, by employing the comparison inequality (Lemma~\ref{lemma:lower_bound1}), we obtain the following:
\begin{align*}
    \inf_{\hat{g}} \sup_{(\eta, P_X) \in \mc{P}'} \pr \lp R_\lambda\lp \hat{g} \rp - R_\lambda \lp ^*\rp \geq c_4 n^{-\beta(1+\alpha_0)/(D + 2\beta)} \rp & \geq \frac{1}{4},
\end{align*}
which gives us the required bound:
\begin{align*}
    \inf_{\hat{g}} \sup_{(\eta, P_X) \in \mc{P}'} \mbb{E}\lb R_\lambda (\hat{g}) - R_\lambda (g^*) \rb & \geq \frac{c_4}{4}n^{-\beta(1+\alpha_0)/(D + 2\beta)}. 
\end{align*}

\subsection{Proof of Corollary~\ref{corollary:lower_bound3}}
\label{appendix:lower_bound3}

We prove this statement by using the correspondence between the Bayes optimal solution under the fixed-cost and the bounded-rate abstention regimes. 
For a given $\delta>0$, we cannot directly apply the construction used in the proof of Theorem~\ref{theorem:lower_bound2} because the amount of probability mass contained in the region $P_X\lp |\eta - 1/2|\leq 1/2 - \lambda \rp$ is $\mc{O}\lp n^{-\alpha_0\beta/(D + 2\beta)} \rp$ which for large enough $n$ can be much smaller than a fixed $\delta>0$. Thus the $\lambda$ level sets of the  constructed regression functions in the proof of Theorem~\ref{theorem:lower_bound2} will not correspond to the Bayes optimal solution with rate of abstention bounded by some fixed $\delta>0$. 

This problem can be fixed in the following way. Let $e_5$ denote  a corner point of $\X=[0,1]^D$ other than $e_j$ for $j=1,2,3$ and $4$, and define $Q_5 = \{x \in \X \mid \|x-e_5\| \leq 1/3 \}$. The regression functions constructed in the proof of Theorem~\ref{theorem:lower_bound2} in the previous sections, are such that $\eta_{\vec{\sigma}}(x) = 1/2$ for all $x x\in Q_5$. 
It suffices to re-define the marginal density $p_X$ to depend on $\vec{\sigma}$ in the following way:

\begin{align*}
p_X^{\vec{\sigma}}(x) = 
\begin{cases}
     \frac{w \indicator_{B(\pi(x),\epsilon/4)}(x)}{\text{Vol}\lp B(\pi(x),\epsilon/4)\rp} & \text{ for } x \in Y_1 \cup Y_2 \\
  \frac{1 - \delta}{2\text{Vol}\lp Q_j \rp} & \text{ for } x \in Q_j, \text{ for } j = 3,4 \\
  \frac{\delta - 2\tilde{M}w}{\text{Vol}(Q_5)}& \text{ for } x \in Q_5 \\
  0 & \text{ otherwise.}
\end{cases}    
\end{align*}

Note that for $n$ large enough and the same choice of parameters $\epsilon$, and $w$, we must have $2\tilde{M}w = \mc{O}\lp n^{-\beta\alpha_0/(2D + \beta)}\rp \leq \delta/2$. This implies that $P_X^{\vec{\sigma}} << P_X^{\vec{\sigma}_0}$ for all $\vec{\sigma}$ in $\Sigma = \{-1,1\}^{2\tilde{M}}$ as required by Theorem~\ref{theorem:tsybakov1}.
The rest of the proof follows from the fact that revealing the threshold can only further decrease the lower bound for the bounded-rate setting.

\newpage
\section{Details from Section~\ref{sec:practical}}
\label{appendix:discussion_feasible}

We now describe how we can modify Algorithm~1 for the case when $P_{XY}$ lies in the family $\mc{P}(L, \beta, \alpha_0)$, i.e., the regression function $\eta(\cdot)= 1/2 + \psi\lp \langle w^*, \cdot \rangle \rp$ for a monotonic and invertible $\psi$, with $\psi(0)=0$, and furthermore the~(MA) assumption is true with exponent $\alpha_0$. We first observe that for this problem, it suffices to estimate the vector $w^*$ and the value $\psi^{-1}\lp \lambda \rp$. Furthermore, the problem of learning $w^*$ in $D$ dimensions, can be reduced to $D-1$ problems of learning two dimensional normalized projections of $w^*$ onto two certain two dimensional subspaces (see \cite[\S~3]{chen2017near}). Thus it suffices to modify Algorithm~1 to solve this problem in two dimensions.

\begin{proposition}
\label{prop:computationally_feasible}
Assume that the joint distribution lies in $\mc{P}_1\lp L, \beta, \alpha_0\rp$, and $\eta(\cdot) = \psi \lp \langle w^*, \cdot \rangle \rp + 1/2$. Then, for $n = \mc{O}\lp \frac{1}{\nu^{-1}\lp \epsilon/(D-1)\rp}\rp^{2 + 1/\beta}$ with $\nu$ defined in~\eqref{eq:nu_bn},  a modification of Algorithm~1 returns an estimate $\hat{w}$ such that $\|\hat{w} - w^*\|_2 \leq \epsilon$ with probability at least $1-2/n$. 
\end{proposition}
\begin{proof}
Assume that $\|w^*\|=1$, and denote by $S^1$ the unit ball in $\mbb{R}^2$. Let $\phi:[0,2\pi) \mapsto S^1$ denote the injective mapping which takes angles to points in $S^1$. We can check that $\phi$ is $1-$Lipschitz, which implies that the mapping $\tilde{\psi} \coloneqq \psi\circ \phi$ is also $(L, \beta)$ H\"older and invertible. 

Now, we can apply Algorithm~1 with $\mc{X}=[0,2\pi)$ to construct $\lambda$ and $1-\lambda$ level set estimates of the mapping $\tilde{\psi}$. 
We know that with $n$ labelled samples, Algorithm~1 can obtain an estimate of $\eta$ with pointwise accuracy of $b_n = \mc{O}\lp n^{-\beta/(2\beta + 1)}\rp$. Now, by the definition of the map $\phi$, and the monotonicity of $\psi$, we know that there exist two values $\theta_{\lambda,1}$ and $\theta_{\lambda, 2}$ such that $\tilde{\psi}(\theta_{\lambda, i}) = \lambda$ for $i=1,2$. 
Next with the notation $E_1 = [\lambda-b_n, \lambda + b_n]$ and $E_2 = [1-\lambda - b_n, 1-\lambda + b_n]$, we introduce the term $\nu$ as 
\begin{equation}
    \label{eq:nu_bn}
  \nu(b_n) \coloneqq \max \left \{\sup_{z_1, z_2 \in E_1}|\tilde{\psi}^{-1}(z_1) - \tilde{\psi}^{-1}(z_2)|, \; \sup_{z_1, z_2 \in E_2}|\tilde{\psi}^{-1}(z_1) - \tilde{\psi}^{-1}(z_2)|  \right \}.  
\end{equation}
Note that the uniform continuity of $\tilde{\psi}$ implies that $\nu(b_n) \to 0$ as $n \to \infty$. 
Since by using estimates $\hat{\theta}_{\lambda, i}$ for $i=1,2$, we can construct estimate of $\theta^* = \phi^{-1}(w^*)$, we note that the error in estimating $w^*$ can be bounded by $\|w^* - \hat{w}\| \leq 1.|\hat{\theta}-\theta^*|\leq 2\nu (b_n)$. This implies that after $D-1$ applications of the two dimensional Algorithm~1 (with total number of labelled samples equal to $(D-1)n$), the estimation error $\|w^* - \hat{w}\| \leq \mc{O}\lp (D-1) \nu(b_n)\rp$ with probability at least $1-2(D-1)/n$. This implies that a sufficient number of labels required for estimating $w^*$ with accuracy $\epsilon$ is $\mc{O}\lp \frac{1}{\nu^{-1}(\epsilon/(D-1))}\rp^{2 + 1/\beta}$. 

\end{proof}

Under the additional assumption of continuously differentiable $\psi$, the required number of labels is $n= \mathcal{O}\lp \lp \frac{D-1}{\epsilon}\rp^{2 + 1/\beta} \rp$. 
We note that in obtaining the convergence rate for Proposition~\ref{prop:computationally_feasible}, we did not employ the monotonicity property of the regression function. Further improvement can be obtained by appropriately modifying the algorithm to perform a noisy binary search, similar to \citep{karp2007noisy}.

\newpage 
\section{Details from Section~\ref{sec:discussion}}
\label{appendix:discussion}

\subsection{Improved rates in active setting.}
\label{appendix:discussion_improved_rates}
Suppose that the marginal $P_X$ has a density $p_X$ w.r.t. the Lebesgue measure, and that the density is bounded below by a constant $c_0>0$ almost surely.
This implies that for any set $A \subset \X$, we have $\pr(X \in A) = P_X(A) \geq c_0 \text{Vol}(A)$. 

Here we show that under this assumption, we have $\tilde{D} \leq \max\{0, D - \alpha_0 \beta \}$.

Define $\lambda_j = 1/2 +(-1)^j \gd$ for $j=1,2$, and the set $\X_{\lambda_j}(\zeta_3(r)) \coloneqq \{ x \in \X \mid |\eta(x) - \lambda_j| \leq 12 L(v_1/(v_2\rho))^{\beta} r^{\beta} \}$. Then by the assumption~(MA), we have the following

\begin{align*}
    P_X\lp \X_{\lambda_j}(\zeta_1(r)) \rp \leq C_0 L^{\alpha_0} \lp \frac{v_1 r}{v_2\rho}\rp^{\beta \alpha_0} \leq \tilde{C}_1 r^{\beta \alpha_0} 
\end{align*}
for some constant $\tilde{C}_1>0$ depending on $L, v_1, v_2, \rho, C_0, \alpha_0, \beta$. 
Furthermore, by the additional assumption on $P_X$, for any $x \in \X$ and $r>0$, we have 
\begin{align*}
    P_X\lp B(x, r) \rp & \geq c_0 \text{Vol}\lp B(x,r) \rp = \tilde{C}_2 r^{D}
\end{align*}
for some constant $\tilde{C}_2 >0$ depending on $c_0$ and $D$. 
Thus for $r>0$, the $r$-packing number of the set $\mc{Z}_r \coloneqq \X_{\lambda_1}\lp \zeta_3(r) \rp \cup \X_{\lambda_2}\lp \zeta_3(r) \rp$ can be upper bounded as follows:
\begin{align*}
\tilde{C}_1 r^{\beta \alpha_0} &\geq    P_X\lp \mc{Z}_r \rp \geq M\lp \mc{Z}_r, r \rp \tilde{C}_2 r^{D} \\
\Rightarrow M\lp \mc{Z}_r, r \rp & \leq \frac{\tilde{C}_1}{\tilde{C}_2} r^{-(D-\beta \alpha_0)}.
\end{align*}
Now, by the definition of \emph{near-$\lambda$ dimension} we observe that $\tilde{D} \leq \max \{0, D-\beta \alpha_0\}$.


\subsection{Need for Detectability (DE) assumption.}
\label{appendix:discussion_DE}
The (DE) assumption ensures that in the regions near the threshold values, the marginal $P_X$ does not put arbitrary small mass. 
Without the (DE) assumption, there will exist joint distributions $P_{XY}$, which will place very small $P_X$ mass in a large region of the input space. 
Since  Algorithm~3 uses the empirical measure $\hat{P}_X$ in order to construct the \emph{unclassified} active set $\Xtu$, even with accurate empirical measures $\hat{P}_X$, for some problem instances the size of the unclassified region would be very large.  
Due to this there is a dependence on the ambient dimension $D$ in the convergence rates obtained for Algorithm~3 without the (DE) assumption. 

Consider the following one dimensional example with $\mc{X}= [0,B]$ for some $B>0$.(Figure~\ref{fig:figure1}). Suppose we have constructed the empirical measure $\hat{P}_X$ with a finite number of samples such that $\sup_{x \in \mc{X}}|F_X(x) - \hat{F}_X(x)|\leq s$ for some $s>0$. 
Suppose $P_X$ has a density $p_X$ such that $p_X(x) = a_1$ for $x \in [0,b_1]$, $p_X(x) = \epsilon$ for $x \in [b_1,b_3]$ and $p_X(x) = a2$ for $x \in [b_3,B]$.
Furthermore, let $b_2 \in (b_1,b_3)$ be the point such that $F_X(b_2) = \delta$. 
Since $\epsilon>0$, $b_1, b_3$ are arbitrary, we can select it in such a way to ensure that $b_3-b_1>B/2$ and $(b_3 - b_1)\epsilon<s$.

At any time $t$, Algorithm~3 constructs (upper and lower) estimates of $\gamma$ using the current estimate of the regression function $\eta$. 
As can be seen from Figure~\ref{fig:figure1}, even if $\eta$ were completely known to the algorithm, the estimated thresholds $\hgone \in [l_1, l_2]$ and $\hgtwo \in [g_1,g_2]$. 
Thus in the worst case the unclassified region will contain the interval $[b_1,b_3]$ of length  at least $B/2$. 

By using the polar-coordinate representation, we can extend this example to the general case of $D$ dimensions, in which we can show that the uncertainty region must contain a ball of sufficiently large radius. 
This implies that the packing dimension of the set $\cup_{\xhi \in \Xtu}\X_{h,i}$ will be equal to $D$. 

\begin{figure}[htb]
\label{fig:figure1}
\includegraphics[width=\textwidth]{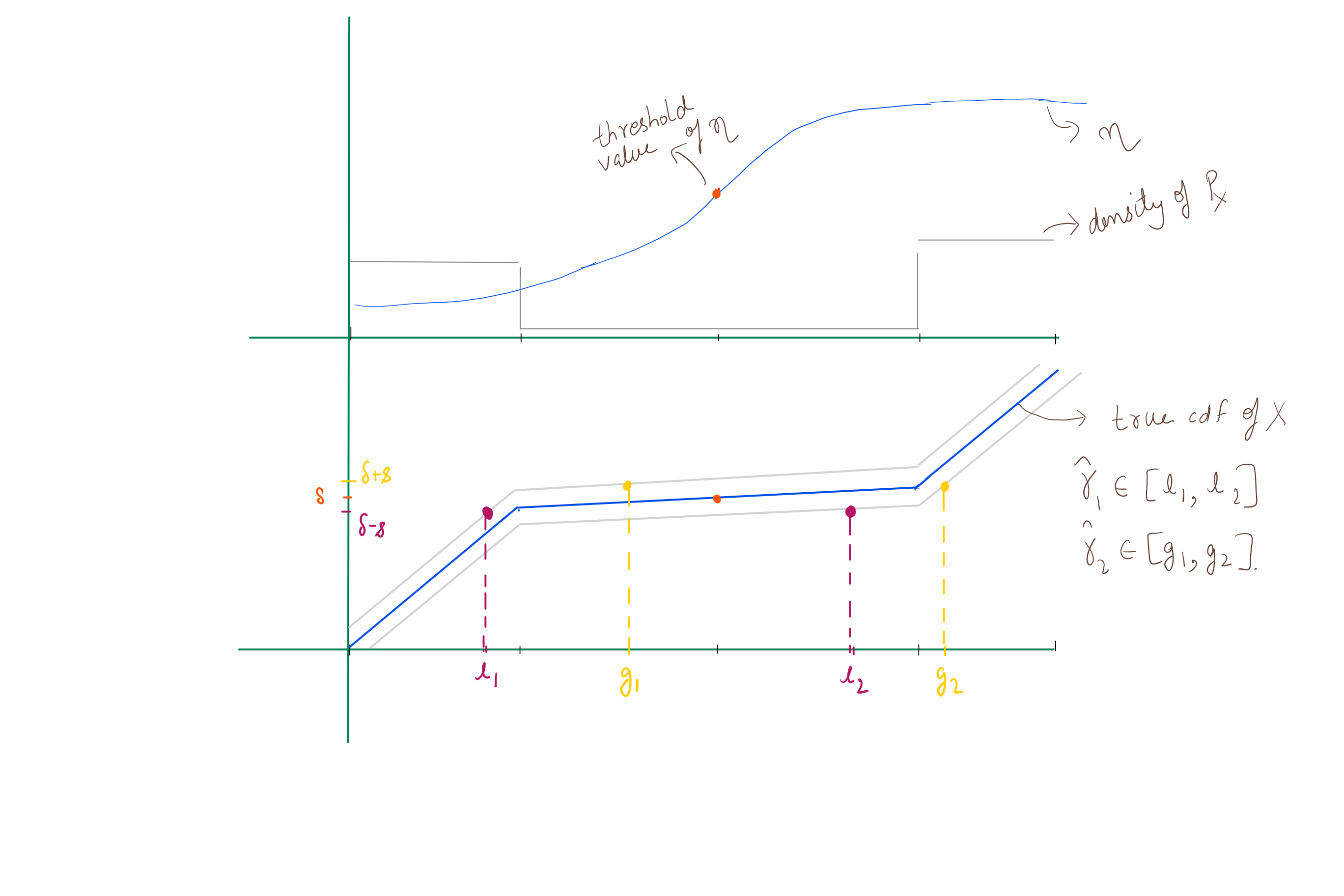}
\caption{The figure shows that if $P_X$ can place arbitrarily small mass near the threshold to be estimated, then even with $\hat{P}_X$ which is uniformly close to $P_X$, the distance between estimated upper and lower bounds on the threshold $\gd$ can be very large.}
\end{figure}

\end{document}